\documentclass[letterpaper, 10 pt, conference]{format/ieeeconf}

\usepackage{amsmath}
\usepackage{amsfonts}

\usepackage{amsthm}
\usepackage[ruled,vlined]{algorithm2e}
\usepackage{mathtools}
\usepackage{tabularx}
\usepackage{graphicx}
\usepackage{subfigure}
\usepackage{enumerate}
\usepackage{float}
\usepackage{url}
\usepackage{verbatim}
\usepackage{booktabs} 
\usepackage{multirow}
\usepackage[linkcolor=black,citecolor=black,urlcolor=black,colorlinks=true]{hyperref}

\graphicspath{{figures/}}
\IEEEoverridecommandlockouts
\newtheorem{theorem}{Theorem}

\begin{document}
	\title{Generating Large Convex Polytopes Directly on Point Clouds}\label{key}
	\author{Xingguang Zhong, Yuwei Wu, Dong Wang, Qianhao Wang, Chao Xu, and Fei Gao
\thanks{X. Zhong, D. Wang, Q. Wang, C. Xu, and F. Gao are with the State Key Laboratory of Industrial Control Technology, Institute of Cyber-Systems and Control, Zhejiang University, Hangzhou 310027, China. (e-mail: {\tt\small zhongstarry@gmail.com, and \{dongwangab, qhwang, cxu, fgaoaa\}@zju.edu.cn }  }
\thanks{Y. Wu is with the Department of Electrical and Systems Engineering, University of Pennsylvania, Philadelphia, PA 19104 USA (e-mail: {\tt\small yuweiwu@seas.upenn.edu} }
}
	\maketitle
	\thispagestyle{empty}
	\pagestyle{empty}
	
	\begin{abstract}
	In this paper, we present a method to efficiently generate large, free, and guaranteed convex space among arbitrarily cluttered obstacles.
Our method operates directly on point clouds, avoids expensive calculations, and processes thousands of points within a few milliseconds, which extremely suits embedded platforms.
The base stone of our method is sphere flipping, a one-one invertible nonlinear transformation, which maps a set of unordered points to a nonlinear space.
With these wrapped points, we obtain a collision-free star convex polytope.
Then, utilizing the star convexity, we efficiently modify the polytope to convex and guarantee its free of obstacles.
Extensive quantitative evaluations show that our method significantly outperforms state-of-the-art works in terms of efficiency.
We also present practical applications with our method in 3D, including large-scale deformable topological mapping and quadrotor optimal trajectory planning, to validate its capability and efficiency.
The source code of our method will be released for the reference of the community.

	\end{abstract}
	
	\IEEEpeerreviewmaketitle
	
	\section{Introduction}
	\label{sec:introduction}
Autonomous robots are emergent in recent years, spreading from aerial delivery to autonomous driving.
In these applications, finding free convex regions among dense obstacles is of vital importance, as they depict navigable space for routing and collision avoidance.

In this paper, we propose a method that efficiently generates large, free and guaranteed convex polytopes directly on point clouds.
Unlike traditional methods that rely on iterative inflation~\cite{gao2020teach} on a discretized map, range query on sophisticated binary space partitioning tree~\cite{fei2018jfr,CheLiuShe2016}, and iterative enumeration and optimization~\cite{deits2015computing}, our method highlights its generalizability, effectiveness and efficiency in three-folds.
Firstly, it guarantees the generated volumes are convex and free, making it seamlessly applicable to mainstream constrained safe trajectory generator.
Secondly, it operates directly and solely on point clouds.
Therefore, it saves much computational overhead from maintaining a post-processed map and retains the most fidelity of raw measurements.
Finally, it makes no assumptions on convexity and structure of environments, thus suits noisy real-world sensors and can be plug-in-and-use in versatile applications.
	
	Our method can be divided into two steps, the modeling of free regions among obstacles and the construction of convex polytopes.
Given an unordered point cloud, we firstly map all points into a nonlinear space by a transformation named sphere flipping~\cite{katz2005mesh}, to flip interior points out.
After this transformation, we calculate the convex hull of these wrapped points and inversely map vertices of the hull back to the original cartesian space.
Using these vertices, we can obtain an obstacle-free polytope, which is proved to be star convex (see Sec.~\ref{sec:methodology}) and is regarded as a representation of free space.
Then, utilizing the special properties of star-convexity, we easily modify the polytope to be strict convex and ensure that it contains no obstacle point as well.

\begin{figure}[t]    
	\centering
	{\includegraphics[width=0.99\columnwidth]{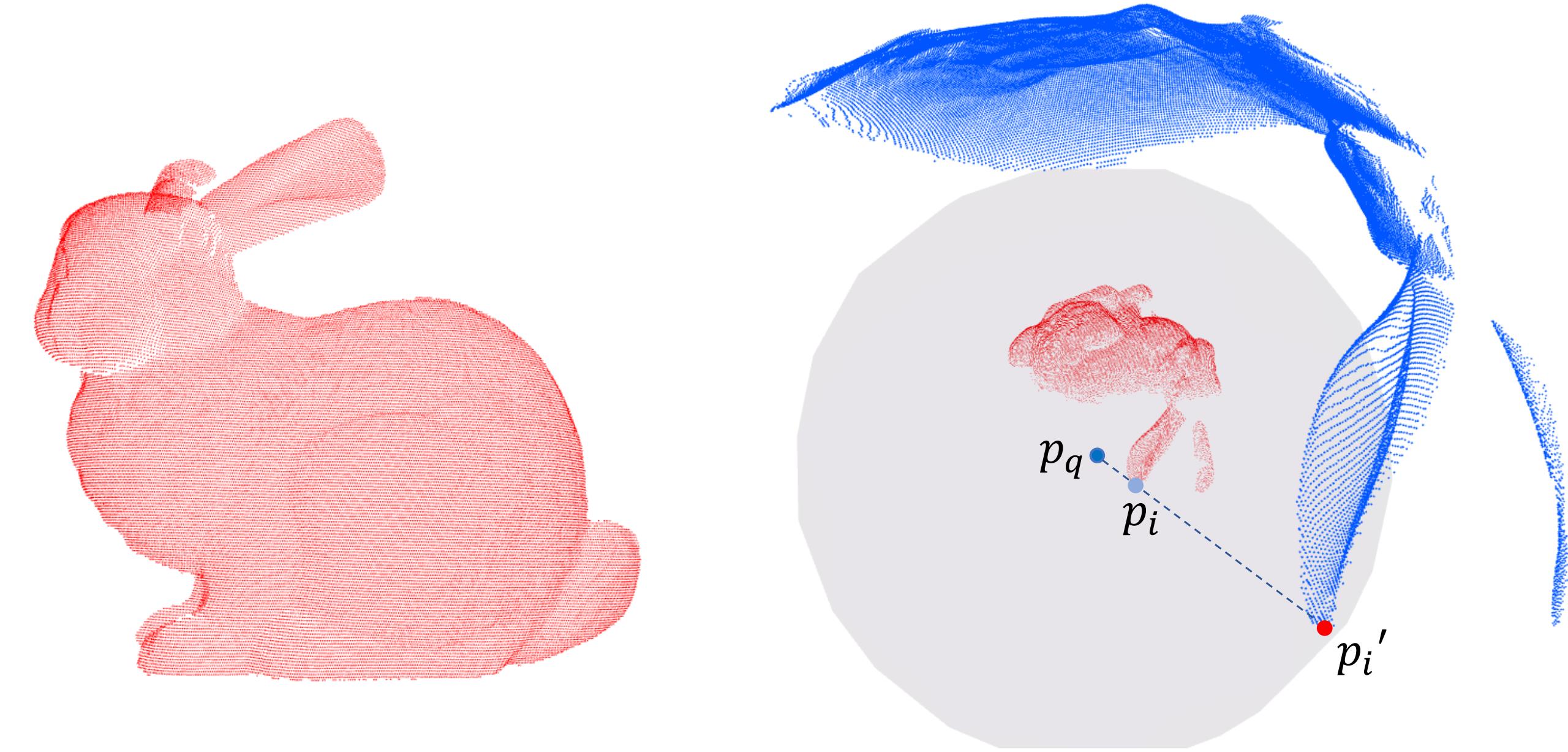}}
	\caption{\label{fig:rabbit} Original 3D point cloud (red) and point cloud after sphere flipping (blue) are shown. $p_q$ named query coordinate is the center of this sphere (grey), $p_i$ and $p_i'$ are a pair of points before and after the mapping }
\end{figure}

Benefiting from the sphere flipping transformation, our method converts the problem of finding an inscribed convex polytope to finding a convex hull, which is well-studied.
In this paper, our method avoids expensive operations and calculations and can take only milliseconds for a point cloud with thousands of points.
Compared with state-of-the-art works, the proposed method finds 3D convex free polytopes with comparable quality and much lower computational overhead and is easily extendable to higher dimensions.
Our method is validated with extensive benchmark tests and applications using synthetic and real-world data.
We summarize our contributions as:
\begin{itemize}
\item A general algorithm to obtain collision-free star convex polytopes given unordered point clouds, by utilizing the sphere flipping transformation.
\item A fast method to trim the star convex polytopes to convex ones while preserving them collision-free.
\item Versatile benchmark tests and demonstrations to show the capability of our method in robotic applications, such as large-scale sparse topological mapping and quadrotor safe trajectory generation.
\end{itemize}

In what follows, we discuss related literature in Sec.~\ref{sec:related_works}.
The proposed method for generating convex free polytopes in point clouds is detailed in Sec.~\ref{sec:methodology}.
Quantitative evaluations showing our method outperform previous works are presented in Sec.~\ref{sec:results}.
Applications utilizing our method for robotic navigation are demonstrated in Sec.~\ref{sec:applications}.
The paper is concluded in Sec.~\ref{sec:conclusion}.

	\section{Related Works}
	\label{sec:related_works}
		
	Obtaining obstacle-free regions is one of the core issues in the field of mobile robots.
Discretizing the configuration space with a grid map or octomap~\cite{Hornung2013OctoMap} to find the obstacle-free area is a natural idea.
Chen et al.~\cite{7487283} propose to generate online a flight corridor consisting of large overlapping 3-D cubes with an octree-based environment representation.
But the corridor it generates is simply axis-aligned, which spans only local space and requires expensive maintenance of the map.
Blchliger et al.~\cite{2017Topomap} present a topological Mapping method named Topomap, which extracts occupancy information from the noisy, sparse point cloud by growing and merging voxel clusters.
It creates a global topological map with a set of convex clusters for robot navigation.
However, building these convex clusters is time-consuming, which prevents it from being used online.
In~\cite{gao2020teach}, the authors propose a method for generating convex polytope from the convex cluster of voxels, and the cluster is obtained by iterative inflation based on a spatial discretized map.
This algorithm requires sophisticated implementation and GPU acceleration to achieve real-time performance, and its time-complexity increases rapidly as the scale of the map.
Some methods~\cite{2010Incremental,2017Building} are proposed to discretize space via the 3D Delaunay triangulation.
Method in~\cite{2010Incremental} carves the volume that violates free-space or visibility constraint.
Similarily, ~\cite{2017Building} also uses Delaunay triangulation, and gives some engineering considerations to make it run in real-time, and finally obtains free space among the point cloud generated by the visual SLAM.
However, the free space obtained by these methods has no guarantee of convexity; thus, it cannot be directly applied to common planning systems.
	
	Some other methods are proposed to obtain convex obstacle-free regions based on the hypothesis of a convex environment.
These algorithms can directly work on raw data obtained by sensors for mobile robots, thereby saving the cost of building a post-processed map.
The representative method IRIS~\cite{deits2015computing} is an iterative algorithm based on the assumption that all obstacles are convex.
While dealing with a point cloud, since every single point is convex, IRIS is applicable if every point is treated as an isolated obstacle. 
Given a starting coordinate, it uses semi-definite programming (SDP) and quadratic programming (QP) to expand the free region iteratively.
Compared to our method in this paper, it costs much more computational resources for generating results with comparable quality.
Liu et al.~\cite{liu2017ral} also propose a method for generating a safe flight corridor (SFC) composed of connected convex polytopes from point cloud.
This method adds linear constraints to the free space by querying the closest point of a given obstacle-free ellipsoid, which is fundamentally similar to a non-iterative version of IRIS.
Besides, this method requires a leading path, which may not be available except for maintaining a grid map.
In~\cite{2017An}, an efficient method for generating an obstacle-free convex polytope directly from the point cloud is proposed.
This work is rather inspiring, for it uses stereographic projection to wrap original points to a nonlinear space, which significantly simplifies the original problem of finding the free space.
However, free space obtained by this method is easily constrained by points near the query coordinate, resulting in polytopes with limited volume.
	
	\section{Convex Polytope Generation}
		\label{sec:methodology}
		Given a query coordinate $p_q \in \mathbb{R}^n$ and the unordered point cloud $ P = \left\{ p_i \in \mathbb{R}^n |1 \leq i \leq N \right\} $, which is considered as the sampling from obstacle surface. Our purpose is to generate a large convex polytope from point $p_q$ and ensure that the polytope does not contain any points in $P$.
	We propose a simple but efficient algorithm to solve this problem, and the method consists of two key steps:
	\begin{enumerate}
	\item Generate a points-free region from a set of unordered points and represent it by a star convex polytope.
	\item Modify the star convex polytope to be strictly convex.
	\end{enumerate}

    \subsection{Generate point-free region}
	Since the points in $P$ have no volume, it is difficult for us to judge whether a certain area is occupied. Traditionally, we often use ray casting or some surface reconstruction methods on point cloud to describe the environment.
	But for unordered point clouds, these algorithms require a lot of computational and storage resources.
	To make our algorithm as fast and lightweight as possible, we avoid using the nearest neighbor search, normal vector calculation, and other complex operations.
	Details of the proposed algorithm are as follows:
	
	We first move the origin to the query coordinate, and then map each $p_i$ in $P$ to $p_i'$ by the nonlinear transformation named sphere flipping below: 
	\begin{equation}
		\label{eq:flipmap}
		p_i' = f(p i) = p_i \cdot (2R - ||p_i||)/||p_i||\\
	\end{equation}
	From (\ref{eq:flipmap}), it is easy to find :
	\begin{equation}
		||p_i'|| = 2R - ||p_i||
	\end{equation}
	where $||\cdot||$ is the 2-norm, $R$ is a user defined parameter which stands for the radius of the sphere. As illustrated in Fig.~\ref{fig:rabbit}, sphere flipping maps every point $p_i$ internal to this sphere along the ray from $p_q$ to $p_i$ to $p_i'$ which is outside of the sphere. 
	
	Obviously, $||f(\cdot)||$ is a decreasing function of norm, which means that the closer a point to the query coordinate is, the farther it will be mapped. Therefore, we can use this mapping to achieve the flip of the internal and external of $P$. The transformed point cloud $P'$ is denoted as: $P' = \left\{ p_i' = f (p i )|p i \in P \right\}$, and the formula for inverse mapping is as follows:
	\begin{equation}
		\label{eq:inversemap}
		p_i = f^{-1}(p_i') = p_i' \cdot ||p_i'||/(2R-||p_i'||)
	\end{equation}
	
	The sphere flipping transformation is critical to our method. 
	For the ease of understanding, we demonstrate it in a 2d case. 
	As shown in Fig.~\ref{fig:sphereflipping}, given two point $P_a$ and $P_b$, we can obtain $P_a' = f(P_a)$ and $P_b' = f(P_b)$ by sphere flipping. By projecting all points on the segment $P_a'P_b'$ back to the original space according to (\ref{eq:inversemap}), a curve $\pi$ between the ray $OP_a$ and $OP_b$ that bulges outward to $P_a'P_b'$ will be generated.  $\pi$'s shape is affected by the parameter $R$, which is derived in detail in ~\cite{katz2007direct}, where this mapping is used to find the visible points directly from unordered point clouds. Suppose that there is a point $P_1'$ in the region $S_l$, the intersection point of ray form $O$ to $P_1'$ and segment $P_a'P_b'$ is $P_L'$, thus the $P_L = f^{(-1)}(P_L')$ is a point on the $\pi$. Due to the decreasing property of sphere flipping, $||P_1|| = ||f^{-1}(P_1')||$ must be larger than $||P_L||$, means that $P_1$ will fall outside the region $S_T$. Similarly, points in $S_r$ will be mapped back into this $S_T$. Thus, we can infer that for the area between rays $oP_a'$ and $oP_b'$, if no image points inside $S_r$, the region $S_T$ is point-free. Furthermore, the triangle $OP_aP_b$ is also free. 
	
	\begin{figure}[t]
		\centering
		\includegraphics[width=0.85\columnwidth]{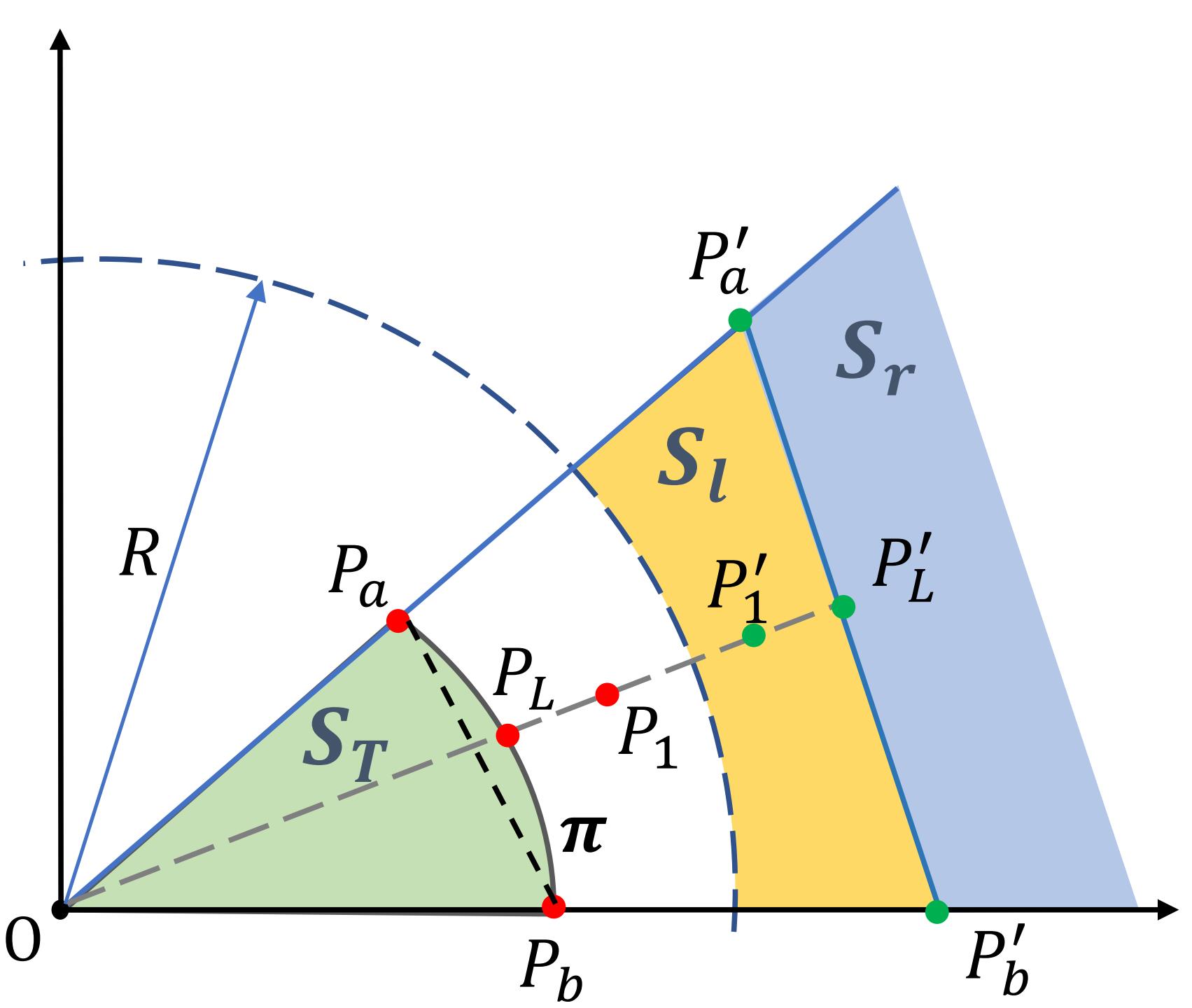}
		\caption{The radius of the dotted circle is the parameter $R$ of sphere mapping, Points in red are in the original space, and the points after mapping are green. The region enclosed by the segment $OP_a$, $OP_b$ and curve $\pi$ is denoted as $S_T$, $S_l$ and $S_r$ represent the left and right side of $P_a'P_b'$.\label{fig:sphereflipping}}
\vspace{-0.5cm}
	\end{figure}
	
	As shown in Fig.~\ref{fig:starconvex}, given the query coordinate $O$ and original point cloud, we assume that there is no hyperplane passing through $O$ can let all other points on the same side of it, which means that point $O$ is wrapped by the point cloud, which is also the assumption of our entire method. Then, We first find the convex hull of points after flipping by the famous quick hull algorithm~\cite{Barber1993The}. $P_a'$ and $P_b'$ are two adjacent vertices of the hull, $P_a$, $P_b$ are two points corresponding to $P_a'$ and $P_b'$ in the original space. According to the property mentioned above, since there is no point on the right side of the line $P_a'P_b'$, the triangle $OP_aP_b$ is point-free. 
    The point-free region $S$ is finally generated by combining the triangles corresponding to each extreme edge.
	
	\begin{theorem}
		\label{thm:star-convex} 
		The region $S$ is a star convex polytope. 
	\end{theorem}
	\begin{proof}
		The region $S$ is star convex if there exists a point in $S$ such that the line segment from $S$ to any point in $S$ is contained in $S$.
		
		Since $O$ is the vertex of all triangles, segment from any point in the area $S$ to $O$ must be contained in $S$.
		Besides, all constraints of $S$ are linear hyperplanes, which means $S$ is a star convex polytope.
	\end{proof} 

    %As $R$ increases, the difference between the points after mapping will become smaller and smaller. (Imagine that when R is very large, all points are approximately distributed on a spherical surface after mapping.) 
     	
	The conclusions obtained can be extended to three and even higher-dimensional spaces easily, and the star convex polytope generated by the above method in 3D environment is shown in Fig.~\ref{fig:yebo}. The only parameter for this method is $R$. Within a specific range, the larger the $R$'s value, the more vertices the star convex polytope will have, and more vertices means higher representation ability.

	\begin{figure}[t]
		\centering
		\includegraphics[width=0.8 \columnwidth]{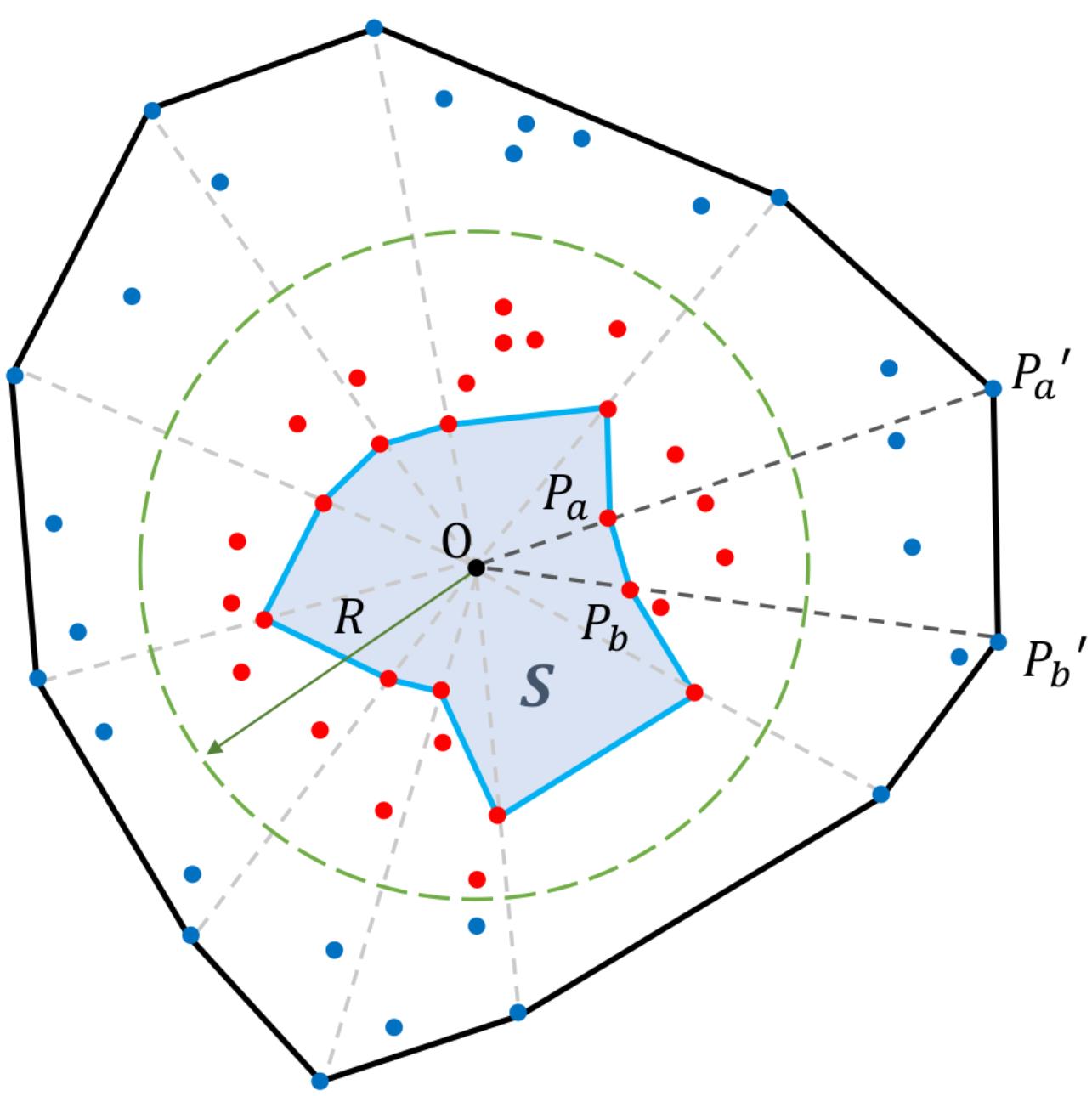}
		\caption{The red points inside the dashed circle are projected to the blue points by one to one correspondence, and point $O$ is the origin and the center of this circle. By finding the convex hull of the projected point cloud, we obtain a Point-free star region $S$..\label{fig:starconvex}}
		%		\vspace{-0.5cm}
	\end{figure}

	\subsection{modification to convex}
	To meet the requirements of general planning algorithm for robotic navigation, we hope to get a sizeable convex polytope from the star convex polytope generated by the previous step. 
	We propose a method that can efficiently modify the star convex polytope generated by the previous step to convex and ensure the convex polytope is still point-free. See the following for details.

	The convex polytope can be expressed by the form: $Ax \leq b$, where $A$ is an $m \times n$ matrix, $b$ is an $n \times 1$ vector and $x \in \mathbb{R}^n$. Each row of this inequality represents a half space: $a_{i1}x_1 + a_{i2}x_2 + \cdot \cdot \cdot + a_{in}x_n \leq b_i$.
	
	We first construct the convex hull $H$ from the vertexes of star convex polytope, and the hull we get may contain some vertices and obstacle points. Thanks to the property of the star convex polytope, each extreme edge of $H$ can form a simplex (a triangle in 2D situation) with origin, as shown in Fig~\ref{fig:pushhull3}.  We only need to select all the star convex vertices inside the simplex and push the extreme edge to the position of the vertex furthest from it, and the new simplex we obtain will be point free.
	
	Redundant constraints can be removed by the double description method~\cite{Fukuda1996Double}.
	For convenience, take the 2D case as an example, the entire process of the modifying algorithm is shown in the Fig~\ref{fig:pushhull}.
	The main process of our algorithm is shown in Alg.~\ref{alg:modificationtoconvex}
	
	\begin{figure}[t]
		\centering
		\includegraphics[width=1.0 \columnwidth]{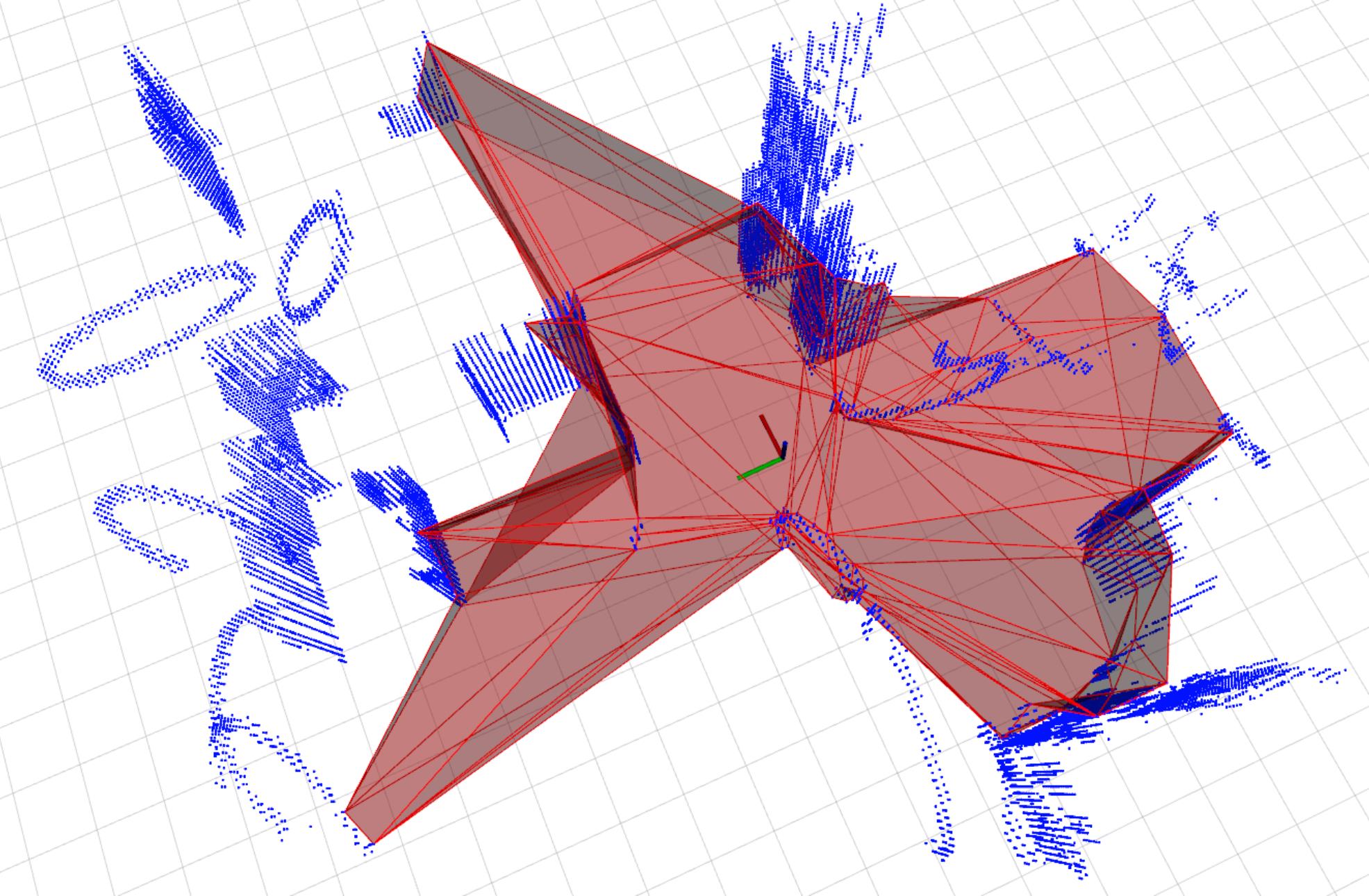}
		\caption{Star convex polytope (red) in cluttered 3D environment.\label{fig:yebo}}
		%		\vspace{-0.5cm}
	\end{figure}
    
    \begin{algorithm}
		\caption{modification to convex}
		\label{alg:modificationtoconvex}
		\KwIn{Vertices $P_v$ of star convex region $S$, The query coordinate $p_q$}
		\KwOut{Convex region in original space, represent as $Ax \leq b$ }
		\Begin
		{
			Construct the convex hull of $P_e$, get the extreme egde set $E$;\\
			\For{$e_i$ in $E$}
			{
				1.Calculate the normal vector $n_i$ of $e_i$;\\
				2.define triangle $T_i$ composed by $e_i$ and $p_q$;\\
				3.Select the point $p_i$ which is furthest from the $e_i$ among all the points of $P_e$ in the triangle $T_i$\;
			}
			\For{$e_i$ in $E$}
			{
				$A(i,:) = n_i^T$\\
				$b_i = n_i^T \cdot p_i$\\
			}
		}
	\end{algorithm}

    \begin{figure}[t]
    	\begin{center}
    		\subfigure[\label{fig:pushhull1}]
    		{\includegraphics[width=0.4\columnwidth]{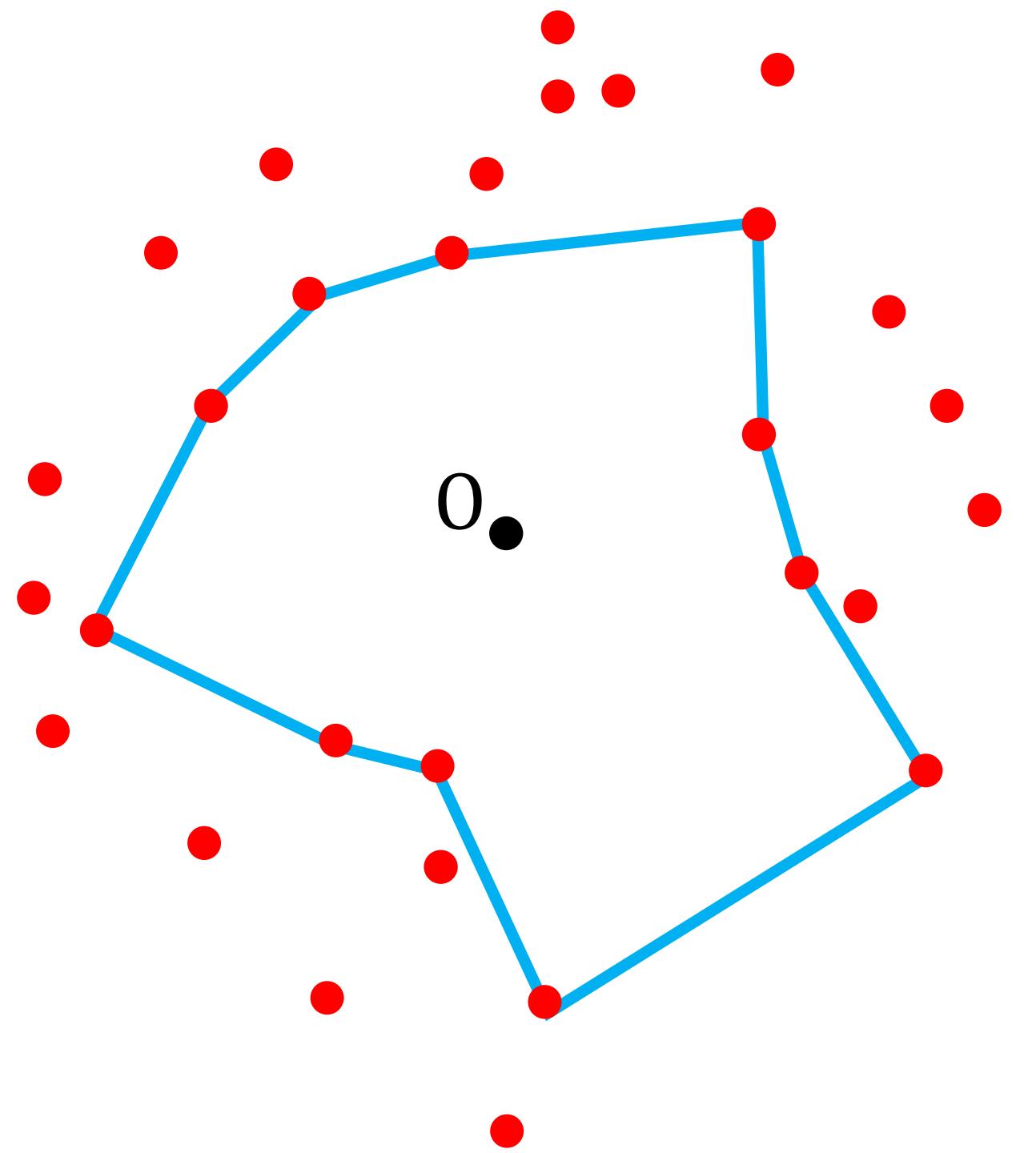}}%
    		\subfigure[\label{fig:pushhull2}]
    		{\includegraphics[width=0.4\columnwidth]{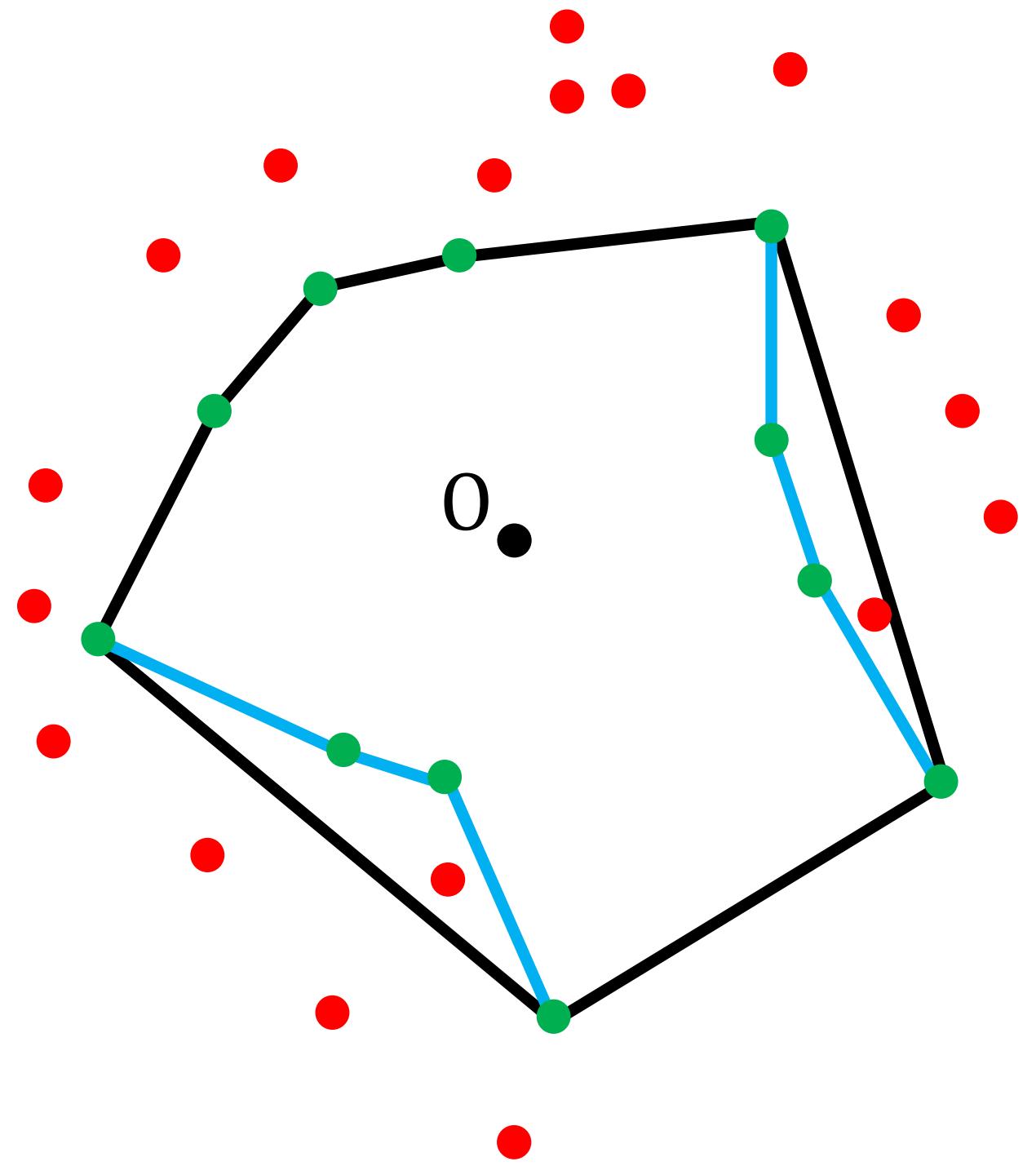}}
    		\subfigure[\label{fig:pushhull3}]
    		{\includegraphics[width=0.4\columnwidth]{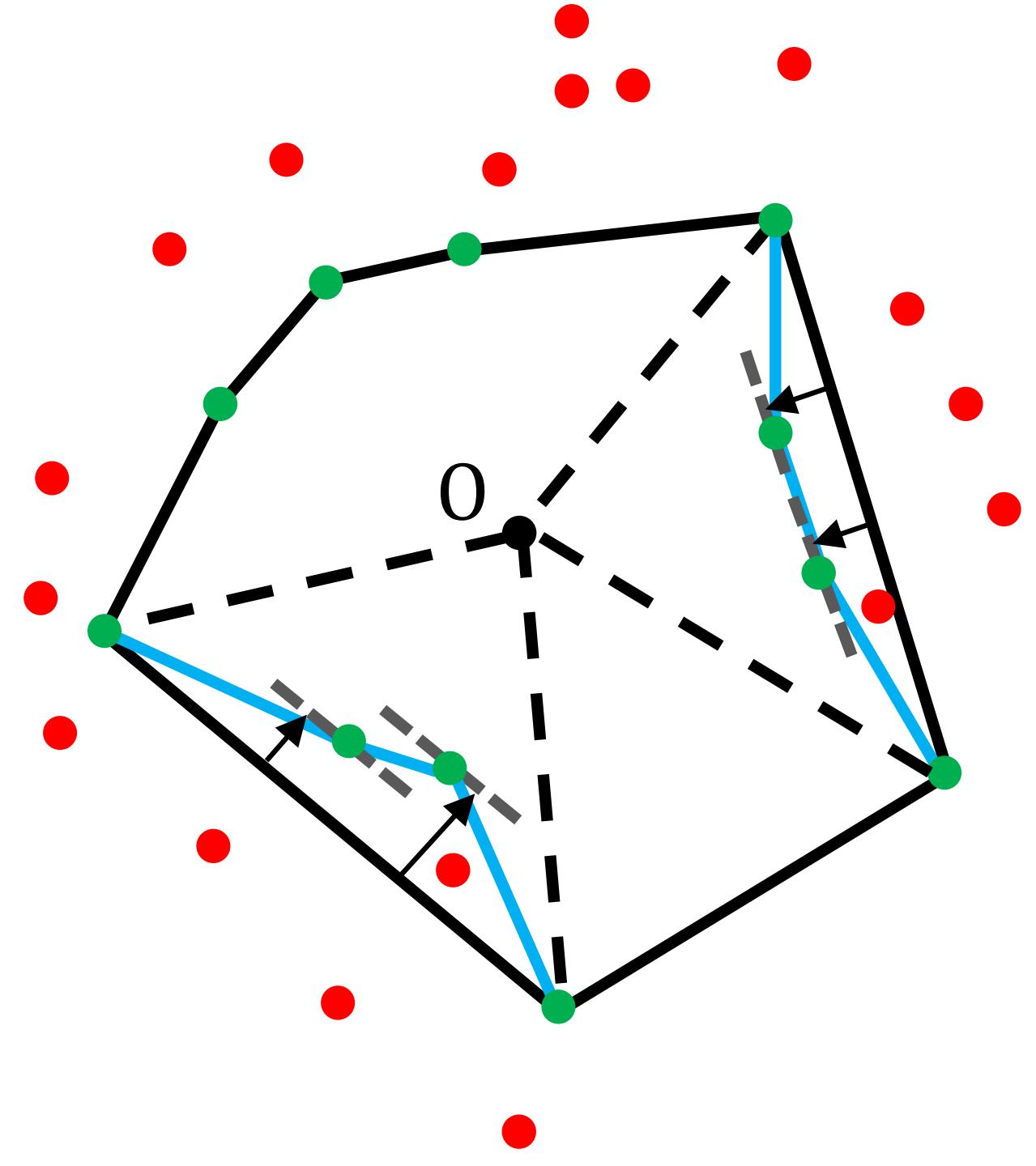}}%
    		\subfigure[\label{fig:pushhull4}]
    		{\includegraphics[width=0.4\columnwidth]{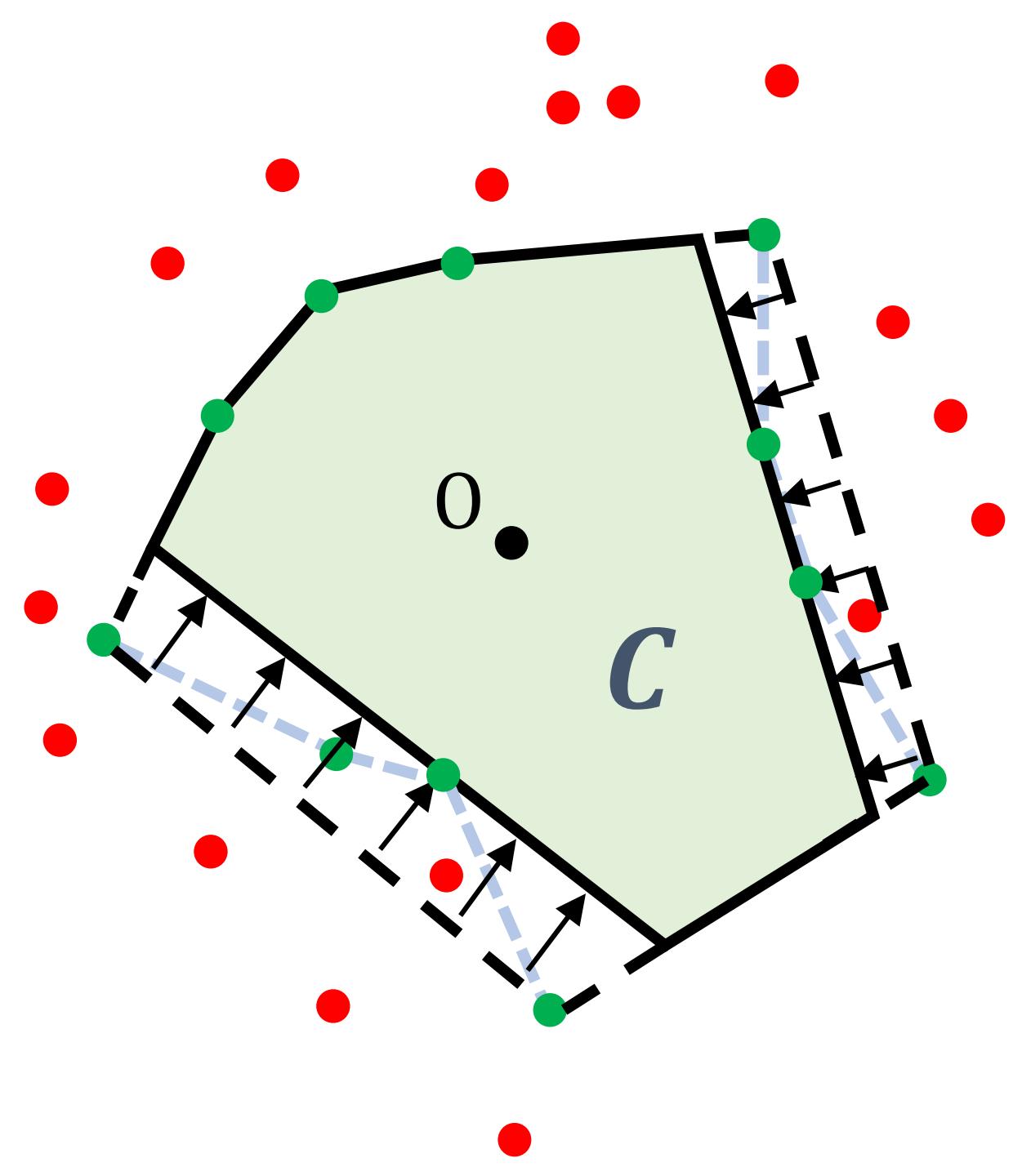}}
    	\end{center}
    	\caption{\label{fig:pushhull} (a) Point free star convex polytope. (b) Construct convex hull. (c) Search for the vertices contained in this convex hull. (d) 
    		Push the extreme edge as far forward as possible along its normal vector.}
    \end{figure}

	This method can not guarantee the convex region obtained is the largest. Still, because the number of edges is small, mostly dozens in three-dimensional space, thus the double description algorithm can get the result very fast. In actual experiments, we found that the time required of modification generally does not exceed 0.5 milliseconds.	
	
	Our algorithm is computationally efficient and easy to implement. The most costly process of our algorithm is to construct the convex hull of points after sphere flipping. In 2D or 3D scenes, the average complexity of convex hull algorithm is $nlog(n)$. Another advantage of our method is that the modified convex polytope must contain query coordinate. That is to say, in robotic motion planning applications, we can directly generate convex obstacle-free region from the position of range sensors without considering whether this region is reachable for robots.

	\section{Quantative Evaluations}
	\label{sec:results}
The proposed method is implemented in C++ 11, and will be released as an open-sourced ros-package~\footnote{Source code will be released at \url{https://github.com/ZJU-FAST-Lab/Galaxy} after the acceptance of this letter}.	 
To verify the effectiveness of our method, we test the algorithm using point clouds collected by Lidar in real environments and pointclouds randomly generated according to specific rules. 
The volume of the generated convex polytope and the running time are used as evaluation criteria to compare with the state of art method. 
IRIS~\cite{deits2015computing}, which also generates a convex point free region with a point cloud and a query coordinate as input.
	
	 IRIS is an iterative method, with each iteration divided into two steps:
\begin{enumerate}
\item Constantly search for the point closest to the ellipsoid, add linear constraints one by one to generate a convex polytope.
\item Use semi-definite programming to calculate the largest inscribed ellipsoid of the convex polytope. This ellipsoid is used for the next iteration.
\end{enumerate}

IRIS ensures that the volume of the ellipsoid converges to the local optimal solution, but it is difficult to calculate the free space in real-time due to the huge amount of calculation in the semi-definite programming.
In this letter, all tests are performed on a consumer-grade notebook computer with an i7-6700HQ CPU.
We compare the results calculated by the first iteration of IRIS (without calculating the ellipsoid), the complete results of IRIS, and the results of our proposed method, as detailed below.
The method proposed by Liu~\cite{liu2017ral} is similar to the first step of IRIS's first iteration, and it requires a prior leading path to initialize the direction of the long axis of the ellipsoid.
Therefore, we consider that there is no need to compare with this method.

\begin{figure}[t]
	\begin{center}
		\subfigure[\label{fig:Lidar_b}]
		{\includegraphics[width=0.9\columnwidth]{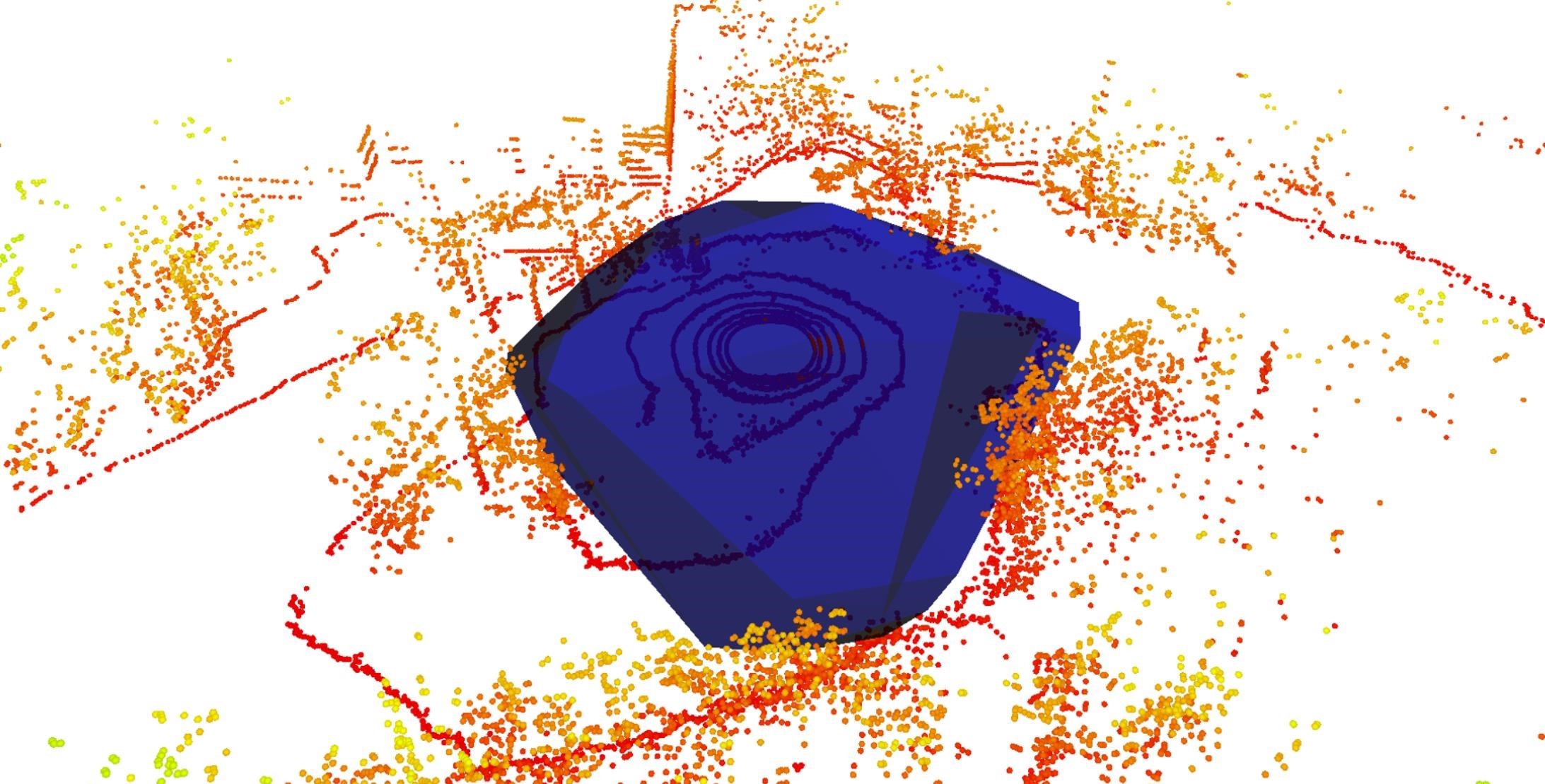}}
		\subfigure[\label{fig:Lidar_g}]
		{\includegraphics[width=0.9\columnwidth]{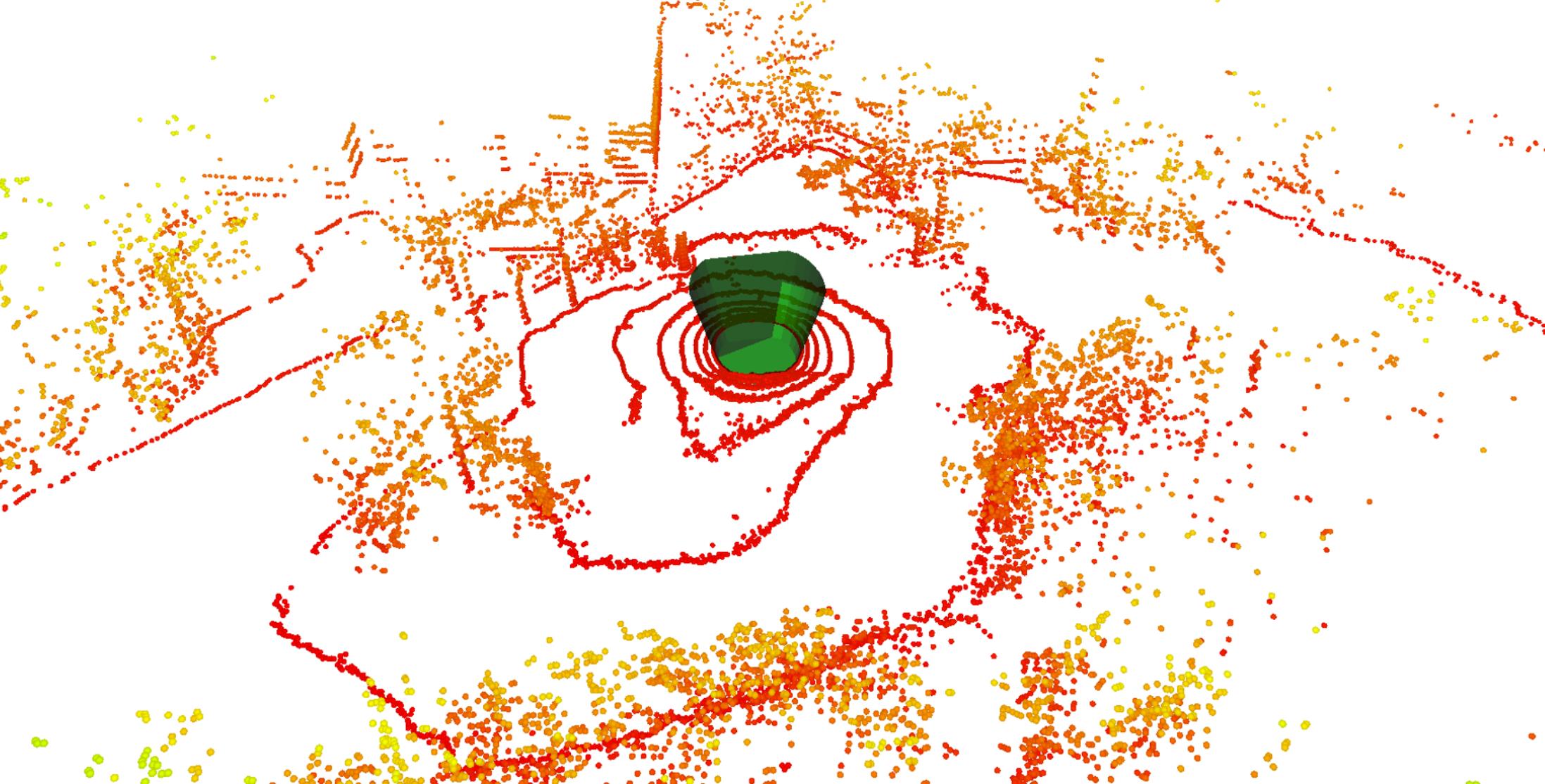}}
		\subfigure[\label{fig:Lidar_r}]
		{\includegraphics[width=0.9\columnwidth]{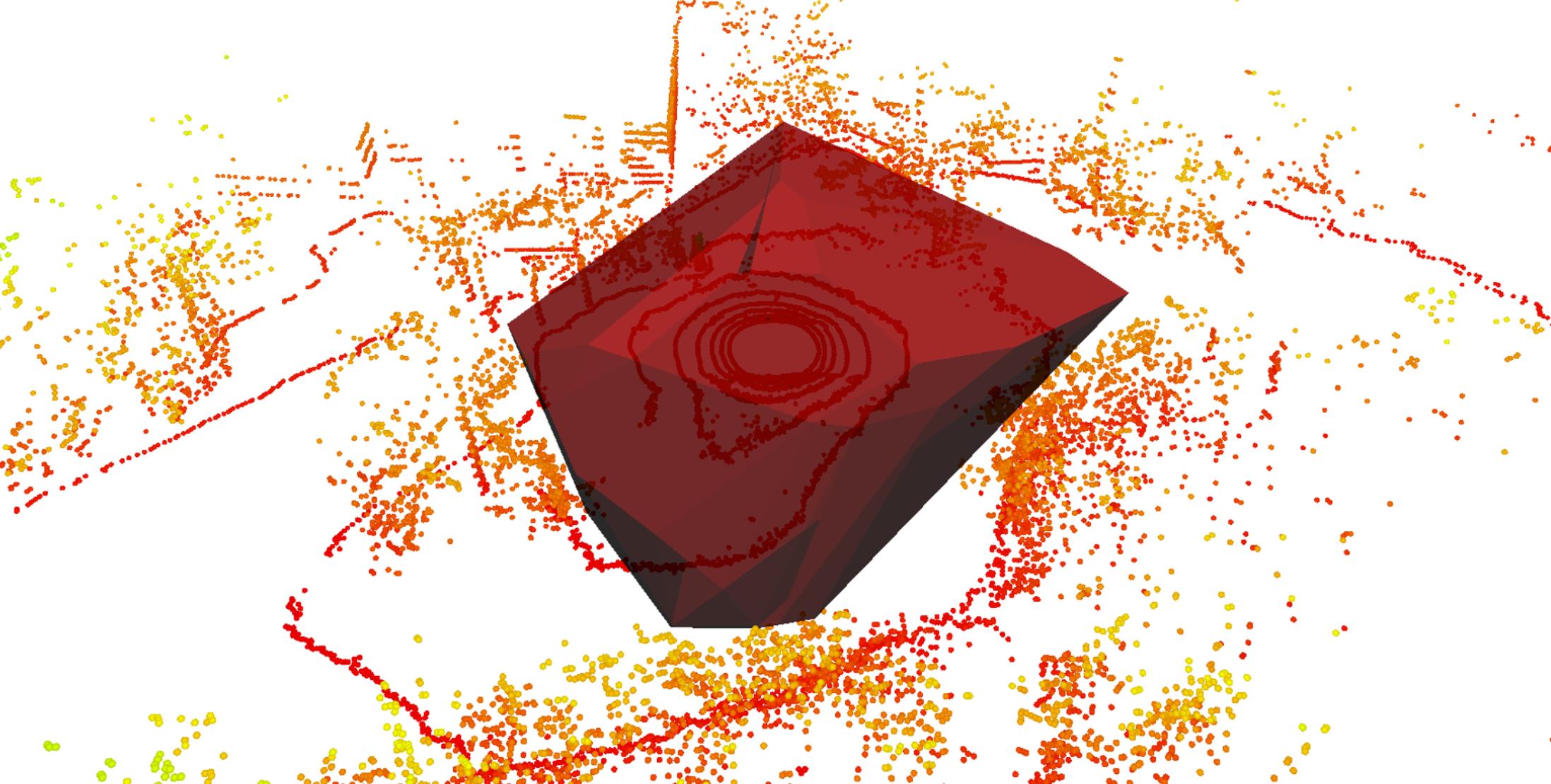}}
	\end{center}
	\caption{\label{fig:Lidar}(a) The final result of IRIS. (b) The result of Non-iterative version of IRIS. (c) The result of Ours.}
\end{figure}

\subsection{Test on real Lidar data set}
We use the point cloud data acquired by a 16-line Lidar with the SLAM framework LIO-SAM~\footnote{\url{https://github.com/TixiaoShan/LIO-SAM}}~\cite{liosam2020shan}.
Our purpose is to verify that the proposed algorithm can be easily applied to range sensors working in complex environments. Since Lidar gets a sparse frame of points, and there is usually not enough point in the $z$ direction that can be used as a boundary, IRIS cannot converge in most cases, so we artificially add a bounding box ($20\times20\times3$) to IRIS.
For a fair comparison, the same bounding box is also added to our algorithm.
Totally, 387 frames are tested with an average number of points in the bounding box as 22814.
Results are shown in Tab.~\ref{tab:test1} and Fig.~\ref{fig:Lidar}.
It shows that the volume of convex polytope generated by our algorithm is 22\% than  the final result of IRIS, but the time required is only 0.6\% of it.
Moreover, compared to the non-iterative version of IRIS, our method is much better than in both terms of final volume and computing speed.
	
	\begin{table}[h]
		\centering
		\caption{\label{tab:test1} test on lidar data set}
		\begin{tabular}{lcll}
			\toprule
			& non-iterative IRIS & IRIS & Ours\\
			\midrule
			Average volume & 37.3002 & \textbf{645.107} & 499.022\\
			Average time(ms) & 145.662 & 1279.66 & \textbf{7.22958} \\
			
			\bottomrule
		\end{tabular}
	\end{table}

    \begin{figure}[t]
	\begin{center}
		\subfigure[\label{fig:ball}]
		{\includegraphics[width=1.0\columnwidth]{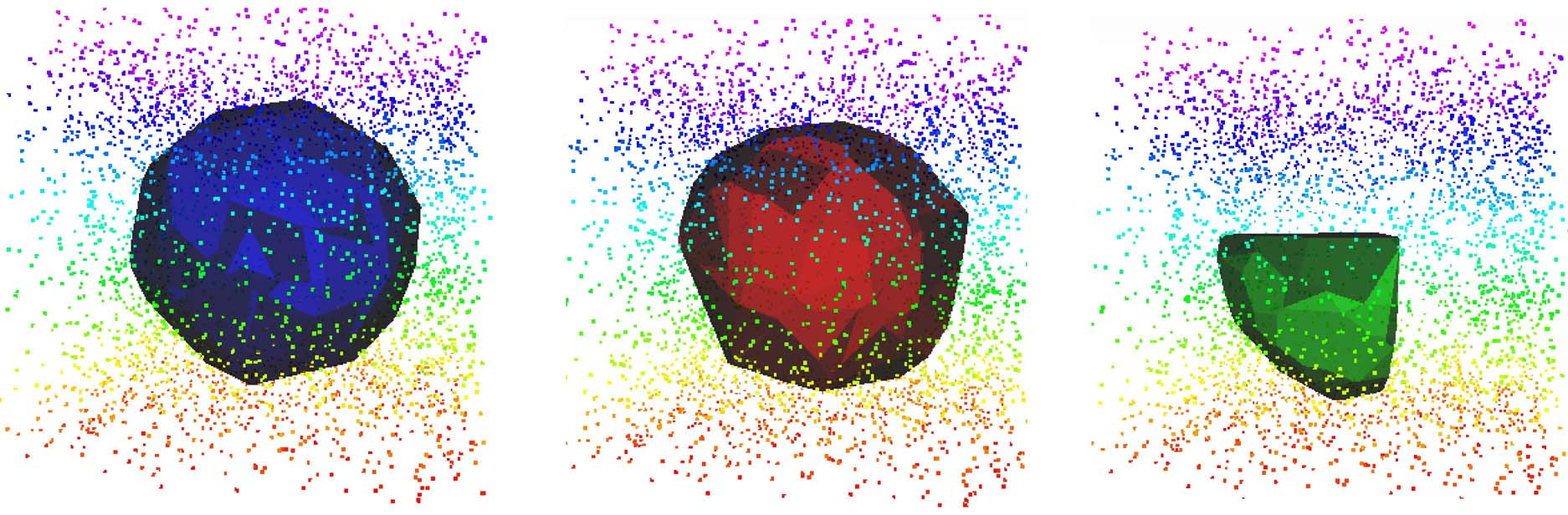}}
		\subfigure[\label{fig:rect}]
		{\includegraphics[width=1.0\columnwidth]{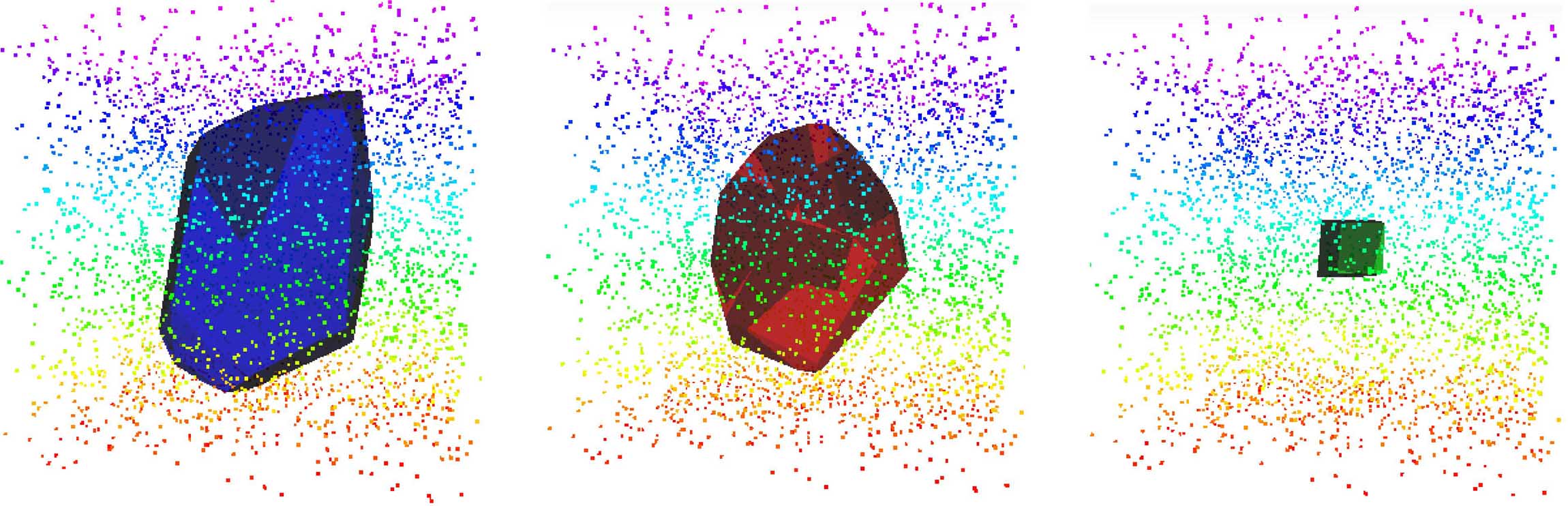}}
		\subfigure[\label{fig:cross}]
		{\includegraphics[width=0.96\columnwidth]{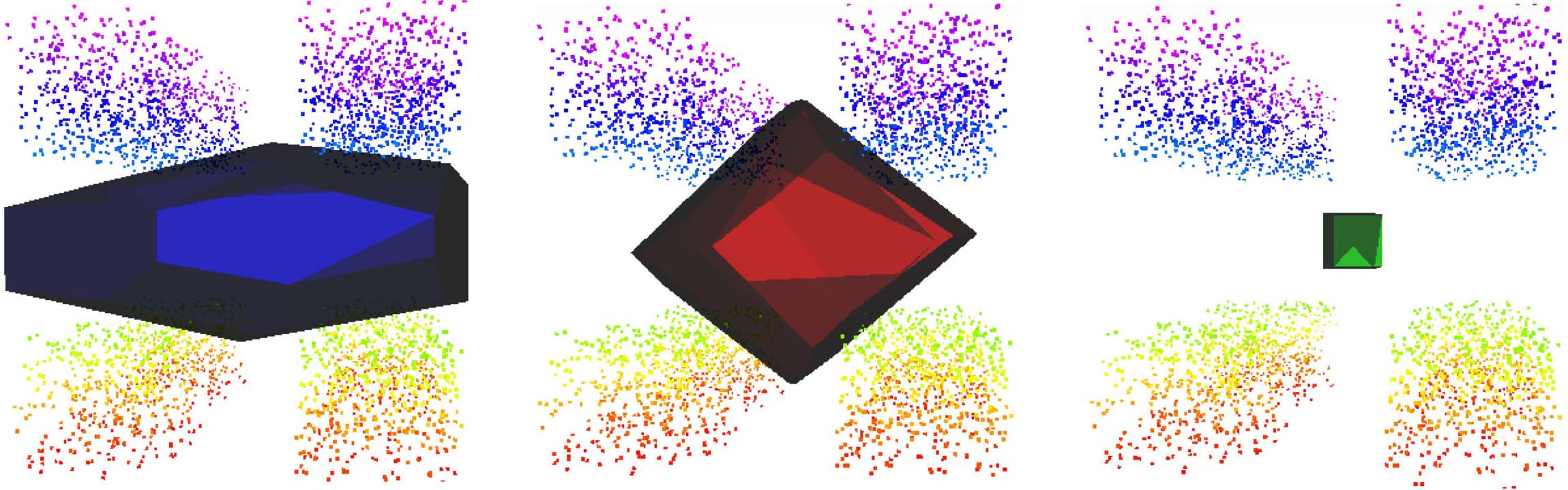}}
	\end{center}
	\caption{\label{fig:random} (a) Ground truth is a sphere (b) Ground truth is a cuboid (c) Ground truth is a cross shape composed by two cuboids. (The blue polytopes are generated by IRIS, the red ones are generated by our method and the green ones are results of non-iterative version of IRIS.)}
    \end{figure}
	
	\subsection{Test on randomly generated data}		
	In order to better reveal the characteristics of our algorithm, we randomly generate points in a cubic space with a point-free area located at its center.
In this test, the number of points is fixed as 3600, and the shape of the free area is sphere, cuboid, and crossed cuboids.

In the three testing scenarios, the results of proposed algorithm and IRIS are shown in Tabs.~\ref{tab:test2} and Figs.~\ref{fig:random}, respectively.
Overall speaking, our algorithm shows superior performance in all these cases compared with the other.

	\begin{table}[h]
	\centering
	\caption{\label{tab:test2} test on random data}
	\begin{tabular}{lllll}
		\toprule
		&  & non-iterative IRIS & IRIS    & Ours   \\
		\midrule
		Sphere                & Average volume & 903.088            & \textbf{1601.98} & 1512.11 \\
		& Average time(ms) & 42.184             & 825.732 & \textbf{7.82334} \\ 
		\midrule
		\multirow{2}{*}{cuboid} &   Average volume & 81.0409            & \textbf{532.232} & 350.213 \\
		& Average time(ms) & 10.1022           & 426.489 & \textbf{2.84751} \\ 
		\midrule
		\multirow{2}{*}{cross} & Average volume & 23.4778            & \textbf{3018.49} & 1720.2  \\
		& Average time(ms) & 11.0267            & 304.786 & \textbf{1.99532} \\ 
		\bottomrule
	\end{tabular}
\end{table}

\begin{figure*}[t]    
	\centering
	{\includegraphics[width=1.9\columnwidth]{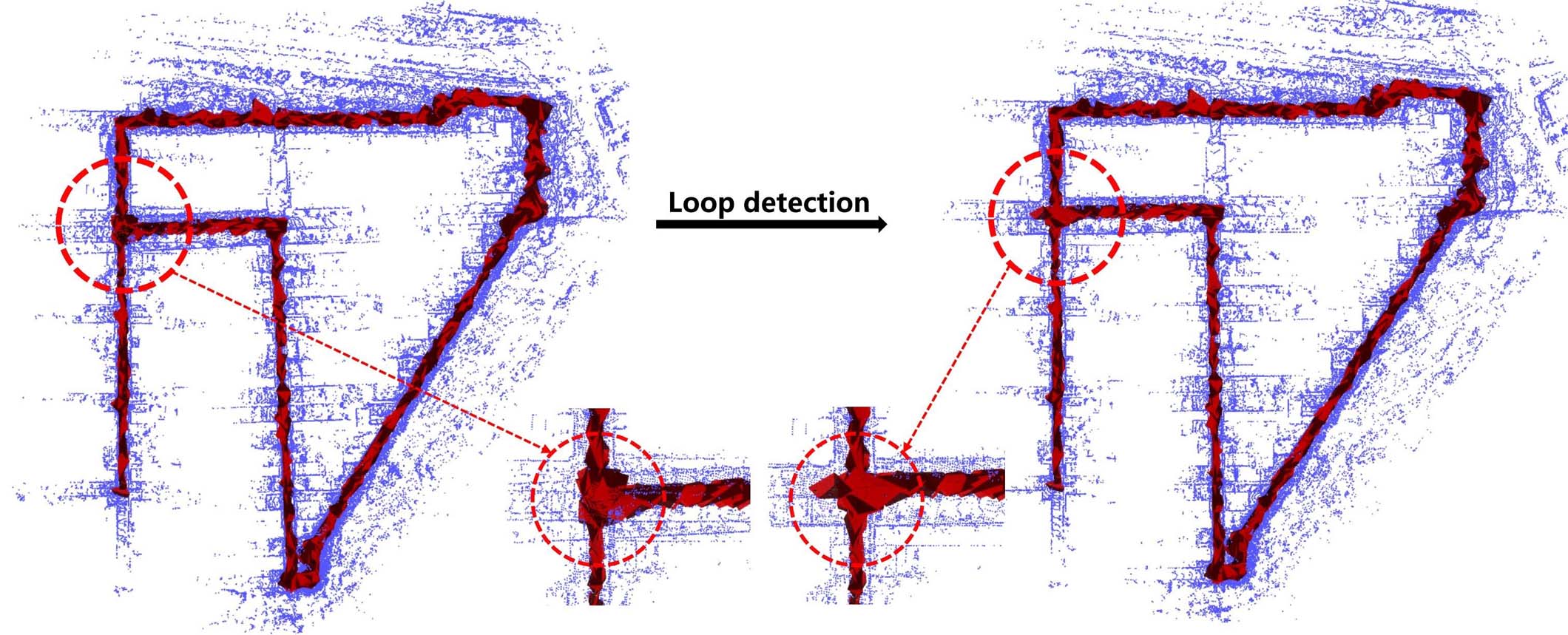}}
	\caption{\label{fig:topo} The visualization of the deformable topological map and the loop correction.The left figure shows the topological map is drifting along with the odometry. The right figure demonstrates that after a loop closure, the topological map is corrected at the same time with the poses and pointclouds. Blue dots are pointclouds mapped by the Lidar SLAM system, red polytopes are convex free polytopes generated by our method. }
\end{figure*}	

\begin{figure}[t]    
	\centering
	{\includegraphics[width=0.99\columnwidth]{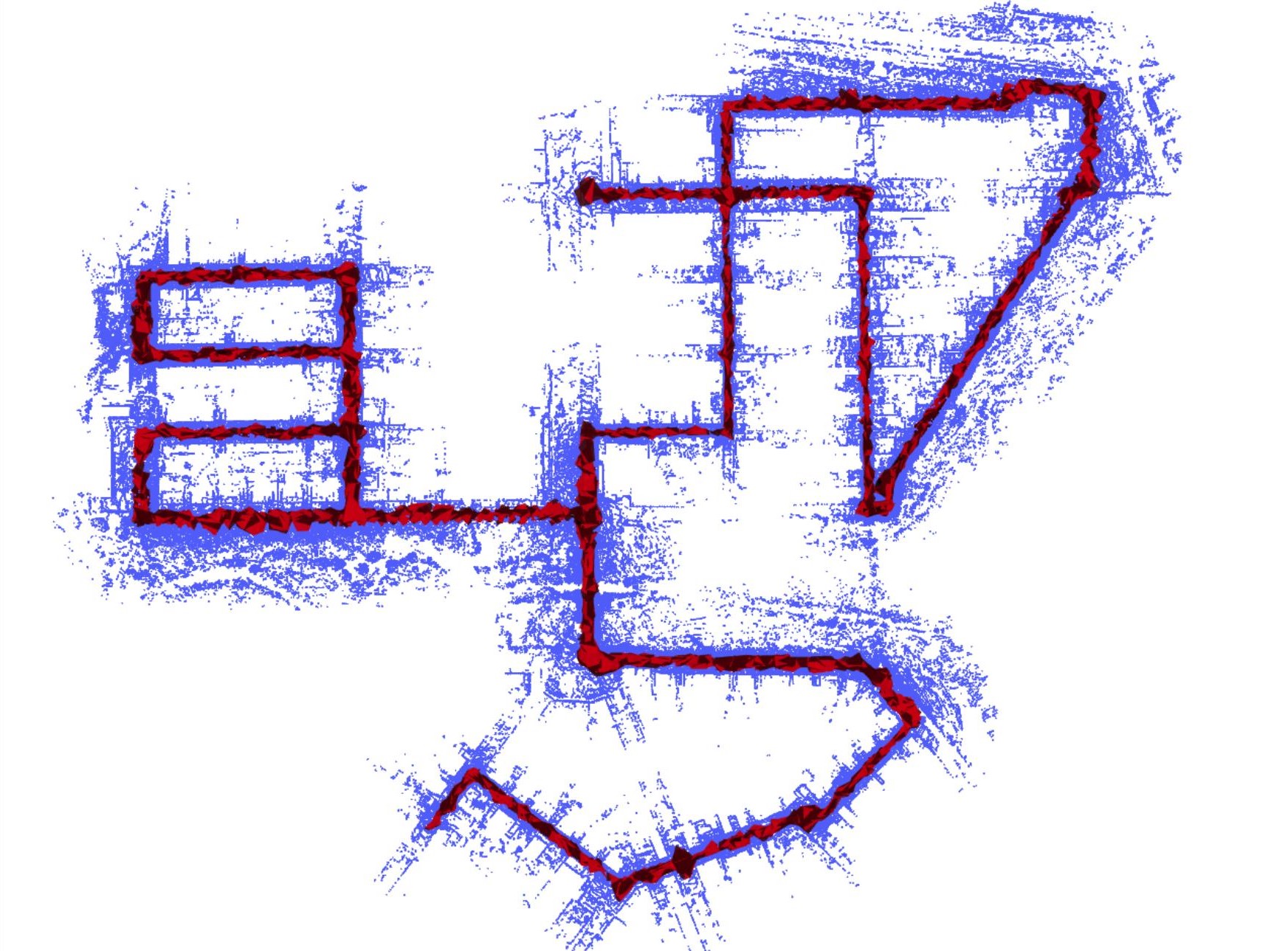}}
	\caption{\label{fig:loop}The overall view of the topological map built by our method on kitti dataset. Markers are interpreted as the same as in Fig.~\ref{fig:topo}. }
\end{figure}

Our method can generate a fairly large convex polytope within a few milliseconds, which fully meets the real-time requirements of robotic mapping and planning.
According to the statistics, in sphere case, the running time of our proposed algorithm is slightly slower than the other two cases.
The reason is that if there are many points close to the query coordinate, these points will be approximately distributed on a spherical surface after the sphere flipping.
This degradation will greatly increase the time required for the convex hull algorithm, but still much faster than IRIS.
Moreover, we note that the volume obtained in the cross case is relatively limited.
That is because, in this situation, much volume is lost in the process of modifying the star convex polytope to convex.

	\section{Applications}
	\label{sec:applications}
	\subsection{Deformable Large-scale Sparse Topological Mapping}	
	\subsubsection{Motivation and Methodology}	
		Mapping large scale environments for repeatable global navigation~\cite{gao2020teach} is of emergent necessity in the robotics and vision community.
With traditional dense mapping technology, as the covered area grows, the map size quickly becomes unacceptably huge, thus preventing its applications in embedded platforms.
Therefore, designing a mapping system that focuses on capturing only free space with sparse data is highly desirable.
Our proposed method naturally suits for this application, as it represents free space centroid among cluttered obstacles with a few parameters of the polytope and relies only on local pointclouds.
Moreover, the polytopes generated by our method attach vertices to obstacle points. Therefore they can be easily globally deformable according to the loop closure of a SLAM system.
Our system features the following advantages:
\begin{itemize}
\item Our method is lightweight yet effective and requires only point clouds as the input.
Therefore, it is applicable to mobile robots for online applications.
\item Since we attach each polytope to a keyframe of the SLAM system when loop closure occurs, polytopes deform along with the pose graph to ensure the global consistency of the topological map.
\item The proposed method exploits the sparsity of environments by representing large free space with dozens of vertices, thus is scaling favorably and is possible to build a map covering a huge area.
\end{itemize}

\subsubsection{Implementation Details}
Here, we also test our method using a 16-line lidar with the LIO-SAM system ~\cite{liosam2020shan}.
Note the proposed method can also be applied to sensors with a limited field of view.
Our method is tested on the Kitti dataset~\cite{Geiger2013Vision}, as demonstrated in Fig.~\ref{fig:topo}.
For the point cloud data obtained in each keyframe, a polytope grouping free space is generated using the proposed method.
Since the vertices of the polytope are also points of this frame, the convex polytope is naturally attached to this keyframe.
As the sensor moves, we generate a new polytope once it leaves the polytope built before, which ensures that convex polytopes consecutively overlaps with minimal redundancy.
In this letter, we highlight the capability of automatically wrapping of our method.
After a loop closure occurs, all affected polytopes deform immediately to keep the overall map consistent, which is shown in Fig.~\ref{fig:loop}.
More details can be viewed in the video\footnote{\url{https://youtu.be/dl0VBgGLLA4}}.

\subsection{Quadrotor Safe Trajectory Generation}

\subsubsection{Motivation and Methodology}
Among enormous methods which support online planning, the hard constrained way which finds free navigable space and solves for constrained trajectories has shown its great capability and generality.
These methods usually build a convex solution space by carving a series of convex and collision-free closed geometrical volumes in a confined environment, then generate flight trajectories within it. 
However, finding large and guaranteed convex free space is nontrivial and often dominates the computational overhead in the overall planning pipeline.
The dominant advantage of our method is the efficiency to generate overlapping convex polytopes (flight corridor) with more free space for planning so that it can be widely used for the corridor based optimization of path planning.

\subsubsection{Implementation Details}
Depending on the characters of reference paths, different strategies are applied to generate connected polytopes.
Rather than using linear paths searched by JPS like~\cite{liu2017ral, tordesillas2019faster}, we firstly search a path by kinodynamic A*~\cite{boyu2019ral} and generate polytopes based on discrete points. 
To utilize the time allocation of the front-end trajectory, we trigger the polytope generation when the next position in the reference path is outside of the exiting corridor or the interval is reached to the time threshold. 
Such particular overlapping polytopes we generate alleviates the difficulties of the optimization process hence gives the trajectory adjustable potential to fix a series of polynomial segments to the polytopes.

To further optimize the trajectory, we apply sum-of-squares based planning~\cite{DeiTed1505} that encodes safety constraints to sum-of-squares conditions, therefore, restricting each trajectory segment in the corresponding polytope.
With ordered overlapping polytopes of approximately equal time intervals, the safe region assignment could be heuristically fixed. In practice, we can directly fix each polynomial piece of the trajectory to each polytope to eliminate integer variables which cause most of the computation cost. 

\subsubsection{Analysis and Comparisons}

We present a comparison between our corridor generation method with Liu's method~\cite{liu2017ral} for global path planning. 
In the simulation environment~\cite{boyu2019ral}, we use a $40 \times 20 \times5m^3$ map which randomly sets 200 circles and 200 obstacles. 
To decrease the computational time, Liu's method employs a bounding box and uses the range of local point clouds inside it. To compare with our method at the same level, we take the same scale of the point clouds along with each seed point into our algorithm. 
We firstly search a leading path using kinodynamic A* for both methods to generate corridor then use the corridor as the input of the optimization constraints.
Since Liu's method relies on a line segment to inflate a polytope, we set a small time interval on kinodynamic A* path to get straight lines and then use the same strategy to generate the next polytope.
In Fig.~\ref{fig:planning}, we show trajectories and corridors generated by ours and Liu's method.

\begin{figure}[t]
	\begin{center}
		\subfigure[\label{fig:sikang_path} Liu's method.]
		{\includegraphics[width=1\columnwidth]{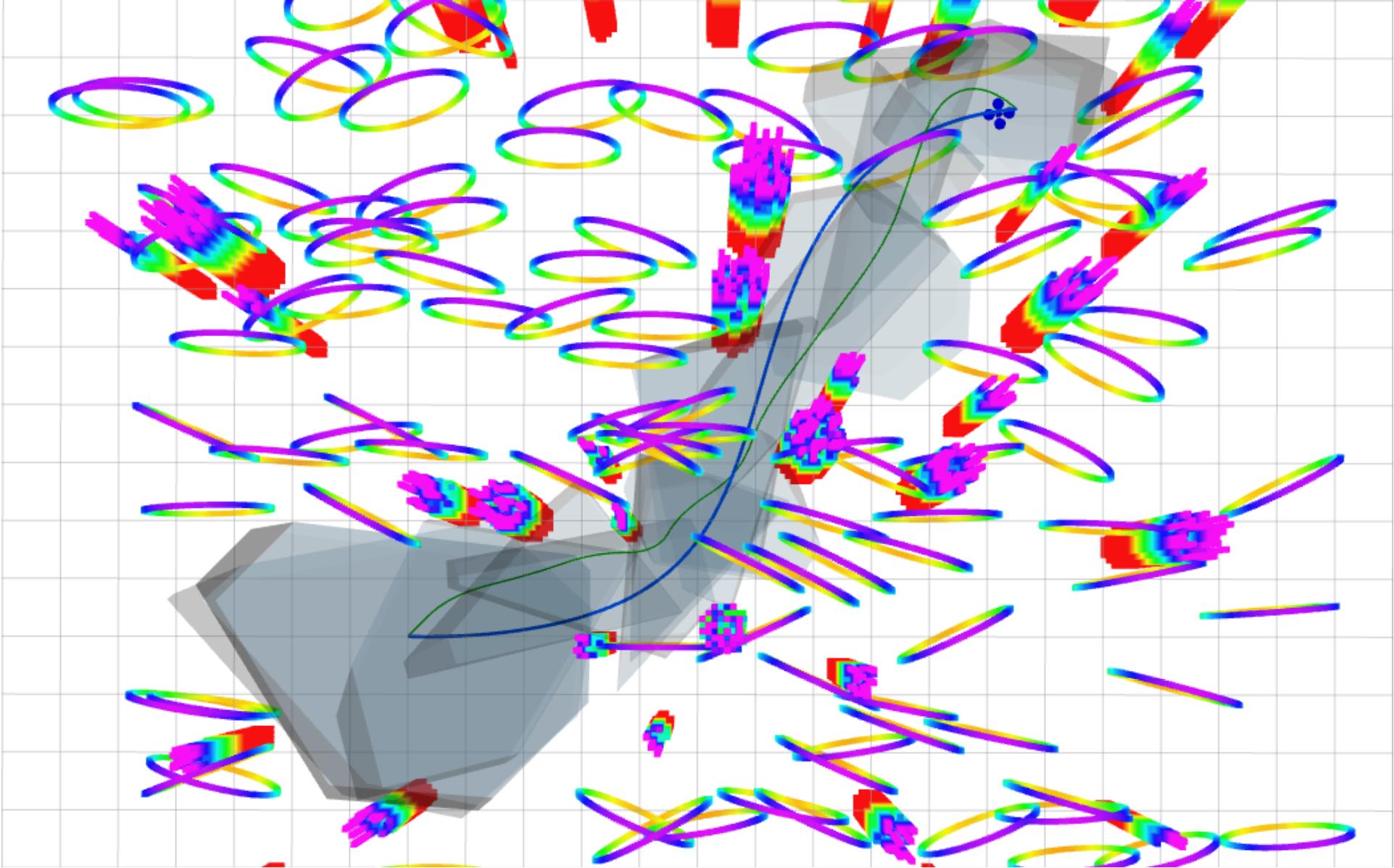}}
		\subfigure[\label{fig:star_path} our method.]
		{\includegraphics[width=1\columnwidth]{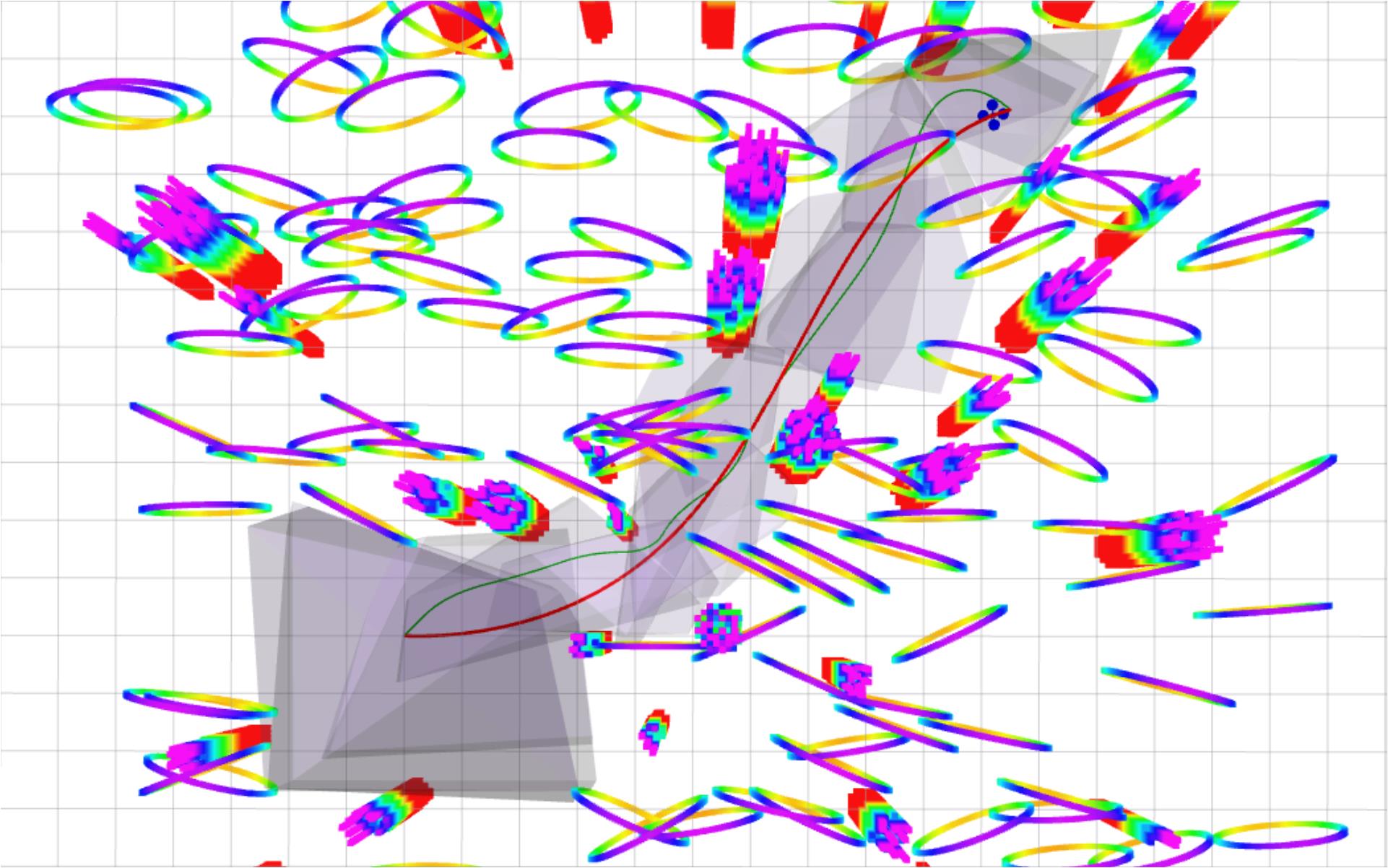}}
	\end{center}
	\caption{\label{fig:planning}Trajectory generation with safe flight corridor. The leading path is a kinodynamic a* path (green) (a) Liu's method and (b) our method to generate overlapping polytopes for trajectory optimization. The blue path is the result using Liu's method while the read one is the optimized trajectory using our method.}
\vspace{-0.5cm}
\end{figure}

\begin{figure}[t]
	\begin{center}
		\subfigure[\label{fig:case1}case 1.]
		{\includegraphics[width=0.45\columnwidth]{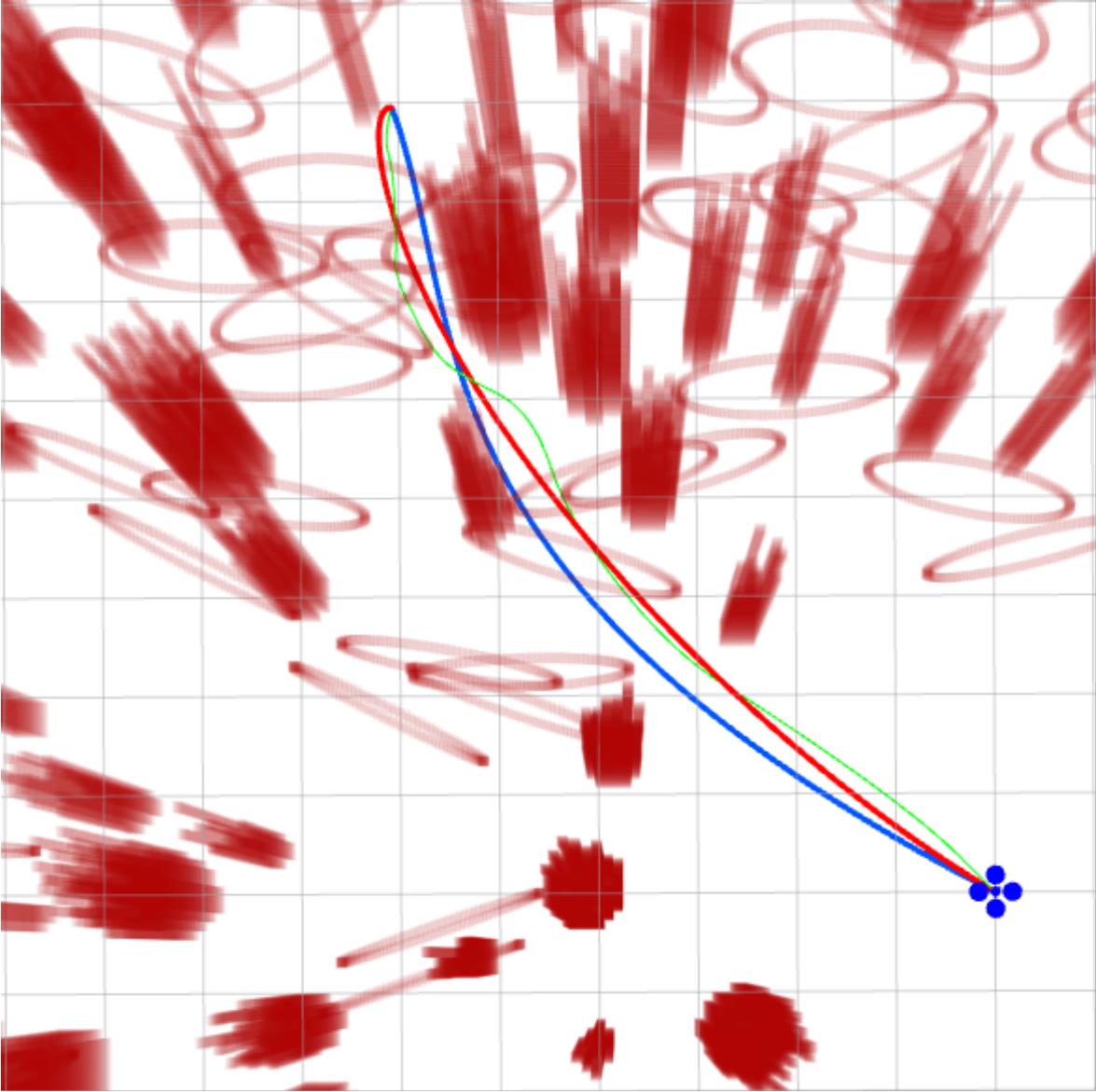}}
		\subfigure[\label{fig:case2} case 2.]
		{\includegraphics[width=0.45\columnwidth]{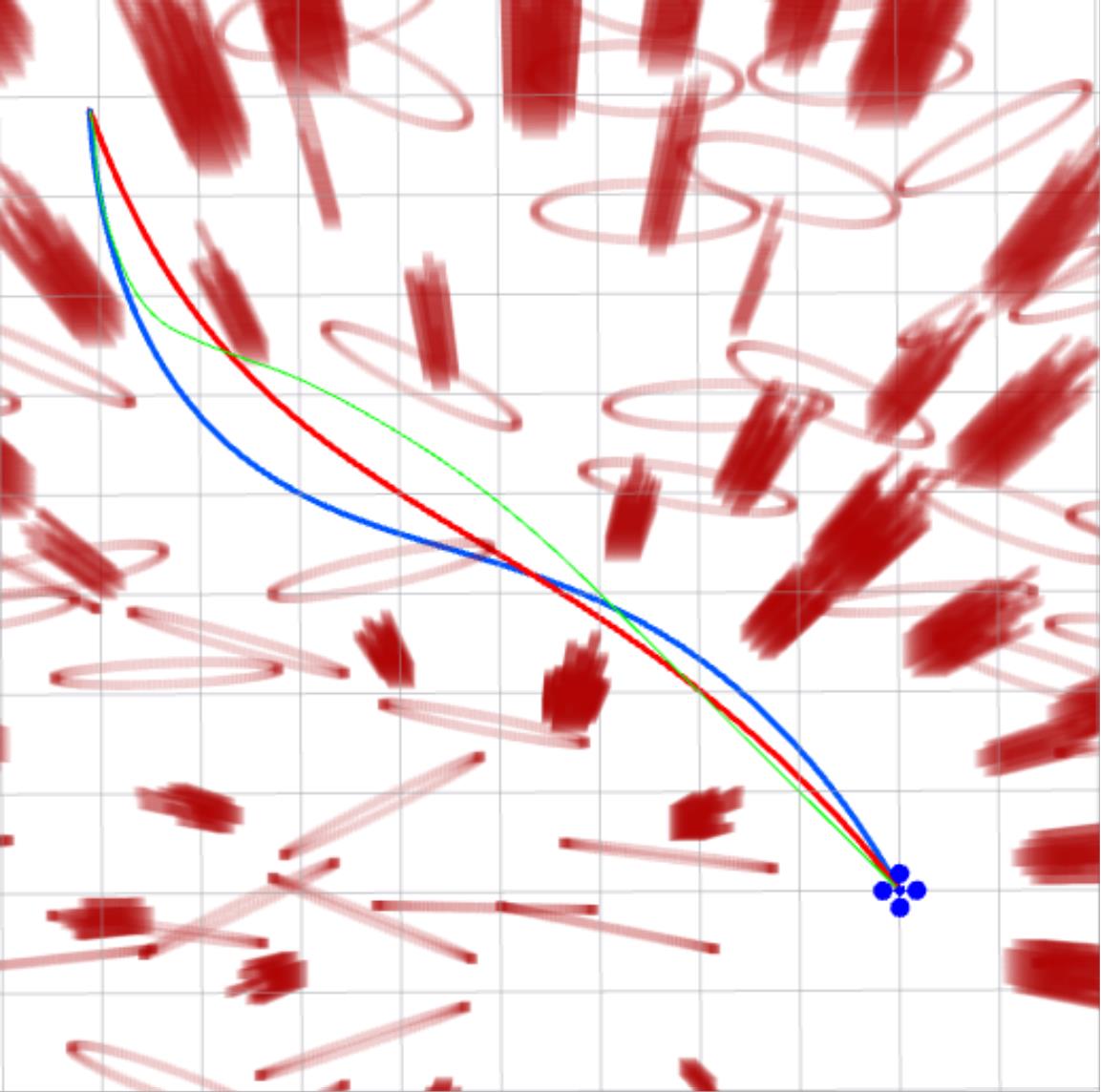}}
	\end{center}
	\caption{\label{fig:traj_comparison} Comparison of trajectory generated by both methods. }
\end{figure}

To compare the quality of the corridors for generating trajectories, we randomly generate 10 simulation maps, and manually set 10 targets in series each time for corridor generation and trajectory optimization.
The results are shown in Tab.~\ref{tab:compare}, which evaluate both methods in the following aspects: the number of polytopes, the number of hyperplanes, and the total computation time to obtain the corridor. 
As validated in the statics, our method significantly saves computational overhead for building a corridor, therefore is suitable for long-term online planning.
Besides, due to precesion of the convex hull finding is tunable, our method can adjust the fedility of the genearted polytopes in complex environments. 
Therefore, in practice, our method can find free polytopes with much fewer hyperplanes.
In general, fewer number of hyperplanes dramatically decreases the number of cone constraints in the sum-of-squares conditions, making the trajectory optimization much more lightweight.

\begin{table}[htbp]
	\centering
	\caption{\label{tab:compare} Comparison of corridors for trajectory optimizations}
	\begin{tabular}{lcll}
		\toprule
		Method & Ours  &  Liu's \\
		\midrule
		number of polytopes  & \textbf{15.30} & 16.69\\
		number of hyperplanes &\textbf{262.74} & 469.96 \\
		time to generate corridor(ms) & \textbf{39.344} & 77.215	\\
		\bottomrule
	\end{tabular}
\end{table}

Each polytope fixes a piece of a polynomial with the same time duration, hence a corridor with fewer polytopes results in less lap time for the whole trajectory and better suits the time allocation found by the leading path. 
Therefore, Fig.~\ref{fig:traj_comparison} shows that our method tends to generate pleasing trajectories with higher quality compared to the other method.
In conclusion, our system has the adaptive capacity to decrease the number of hyperplanes of the corridor while guaranteeing the quality to satisfy online real-time planning.

	\section{Conclusion}
	\label{sec:conclusion}
In this paper, we introduce a novel method for finding large convex polytopes directly on pointclouds.
Our method utilizes sphere flipping to transform original points to a nonlinear space, where the convex hull can be quickly found.
Then, we map the convex hull back to the original Cartesian space and obtain a star convex polytope.
Finally, the star convex polytope is modified to be convex by squeezing local nonconvex vertices.
Compared to previous state-of-the-art works, our method highlights its superior efficiency and scalability, and attain comparable solution quality.
To further strengthen the capability and versatility of our method, we also demonstrate relevant applications in deformable large-scale sparse topological mapping and quadrotor trajectory generation.
In the future, we plan to study the optimal local free space grouping based on clustering.
In this way, the proposed method can be guided by the local spreading of obstacles, thus more easily finding larger polytopes.

\bibliography{references}

\begin{thebibliography}{10}
\providecommand{\url}[1]{#1}
\csname url@rmstyle\endcsname
\providecommand{\newblock}{\relax}
\providecommand{\bibinfo}[2]{#2}
\providecommand\BIBentrySTDinterwordspacing{\spaceskip=0pt\relax}
\providecommand\BIBentryALTinterwordstretchfactor{4}
\providecommand\BIBentryALTinterwordspacing{\spaceskip=\fontdimen2\font plus
\BIBentryALTinterwordstretchfactor\fontdimen3\font minus
  \fontdimen4\font\relax}
\providecommand\BIBforeignlanguage[2]{{%
\expandafter\ifx\csname l@#1\endcsname\relax
\typeout{** WARNING: IEEEtran.bst: No hyphenation pattern has been}%
\typeout{** loaded for the language `#1'. Using the pattern for}%
\typeout{** the default language instead.}%
\else
\language=\csname l@#1\endcsname
\fi
#2}}

\bibitem{gao2020teach}
F.~Gao, L.~Wang, B.~Zhou, X.~Zhou, J.~Pan, and S.~Shen,
  ``{Teach-Repeat-Replan}: A complete and robust system for aggressive flight
  in complex environments,'' \emph{IEEE Transactions on Robotics}, vol.~36,
  no.~5, pp. 1526 -- 1545, 2020.

\bibitem{fei2018jfr}
F.~Gao, W.~Wu, W.~Gao, and S.~Shen, ``Flying on point clouds: Online trajectory
  generation and autonomous navigation for quadrotors in cluttered
  environments,'' \emph{Journal of Field Robotics}, 2018.

\bibitem{CheLiuShe2016}
J.~Chen, T.~Liu, and S.~Shen, ``Online generation of collision-free
  trajectories for quadrotor flight in unknown cluttered environments,'' in
  \emph{Proc. of the {IEEE} Intl. Conf. on Robot. and Autom.}, Stockholm,
  Sweden, May 2016, pp. 1476--1483.

\bibitem{deits2015computing}
R.~Deits and R.~Tedrake, ``Computing large convex regions of obstacle-free
  space through semidefinite programming,'' in \emph{Algorithmic Foundations of
  Robotics XI}.\hskip 1em plus 0.5em minus 0.4em\relax Springer, 2015, vol.
  107, pp. 109--124.

\bibitem{katz2005mesh}
S.~Katz, G.~Leifman, and A.~Tal, ``Mesh segmentation using feature point and
  core extraction,'' \emph{The Visual Computer}, vol.~21, no. 8-10, pp.
  649--658, 2005.

\bibitem{Hornung2013OctoMap}
A.~Hornung, K.~M. Wurm, M.~Bennewitz, C.~Stachniss, and W.~Burgard, ``Octomap:
  An efficient probabilistic 3d mapping framework based on octrees,''
  \emph{Autonomous Robots}, vol.~34, no.~3, pp. 189--206, 2013.

\bibitem{7487283}
{Jing Chen}, {Tianbo Liu}, and {Shaojie Shen}, ``Online generation of
  collision-free trajectories for quadrotor flight in unknown cluttered
  environments,'' in \emph{2016 IEEE International Conference on Robotics and
  Automation (ICRA)}, 2016, pp. 1476--1483.

\bibitem{2017Topomap}
F.~Blchliger, M.~Fehr, M.~Dymczyk, T.~Schneider, and R.~Siegwart, ``Topomap:
  Topological mapping and navigation based on visual slam maps,'' 2017.

\bibitem{2010Incremental}
e.~a. Lovi, David, ``Incremental free-space carving for real-time 3d
  reconstruction.'' \emph{Fifth international symposium on 3D data processing
  visualization and transmission (3DPVT)}, 2010.

\bibitem{2017Building}
Y.~Ling and S.~Shen, ``Building maps for autonomous navigation using sparse
  visual slam features,'' in \emph{IEEE/RSJ International Conference on
  Intelligent Robots and Systems}, 2017.

\bibitem{liu2017ral}
S.~Liu, M.~Watterson, K.~Mohta, K.~Sun, S.~Bhattacharya, C.~J. Taylor, and
  V.~Kumar, ``Planning dynamically feasible trajectories for quadrotors using
  safe flight corridors in 3-d complex environments,'' \emph{IEEE Robotics and
  Automation Letters ({RA-L})}, pp. 1688--1695, 2017.

\bibitem{2017An}
S.~Savin, ``An algorithm for generating convex obstacle-free regions based on
  stereographic projection,'' in \emph{International Siberian Conference on
  Control and Communications}, 2017.

\bibitem{katz2007direct}
S.~Katz, A.~Tal, and R.~Basri, ``Direct visibility of point sets,'' in
  \emph{ACM SIGGRAPH 2007 papers}, 2007, pp. 24--es.

\bibitem{Barber1993The}
C.~B. Barber, D.~P. Dobkin, and H.~Huhdanpaa, ``The quickhull algorithm for
  convex hulls,'' \emph{Acm Transactions on Mathematical Software}, vol.~22,
  no.~4, 1993.

\bibitem{Fukuda1996Double}
K.~Fukuda, ``Double description method revisited,'' in \emph{Combinatorics and
  Computer Science, Franco-japanese and Franco-chinese Conference, Brest,
  France, July, Selected Papers}, 1996.

\bibitem{liosam2020shan}
T.~Shan, B.~Englot, D.~Meyers, W.~Wang, C.~Ratti, and R.~Daniela, ``Lio-sam:
  Tightly-coupled lidar inertial odometry via smoothing and mapping,'' in
  \emph{IEEE/RSJ International Conference on Intelligent Robots and Systems
  (IROS)}.\hskip 1em plus 0.5em minus 0.4em\relax IEEE, 2020.

\bibitem{Geiger2013Vision}
A.~Geiger, P.~Lenz, C.~Stiller, and R.~Urtasun, ``Vision meets robotics: The
  kitti dataset,'' \emph{International Journal of Robotics Research}, vol.~32,
  no.~11, pp. 1231--1237, 2013.

\bibitem{tordesillas2019faster}
J.~Tordesillas, B.~T. Lopez, and J.~P. How, ``Faster: Fast and safe trajectory
  planner for flights in unknown environments,'' pp. 1934--1940, 2019.

\bibitem{boyu2019ral}
B.~Zhou, F.~Gao, L.~Wang, C.~Liu, and S.~Shen, ``Robust and efficient quadrotor
  trajectory generation for fast autonomous flight,'' \emph{IEEE Robotics and
  Automation Letters}, vol.~4, no.~4, pp. 3529--3536, 2019.

\bibitem{DeiTed1505}
R.~Deits and R.~Tedrake, ``Efficient mixed-integer planning for {UAV}s in
  cluttered environments,'' in \emph{Proc. of the {IEEE} Intl. Conf. on Robot.
  and Autom.}, Seattle, Washington, USA, May 2015, pp. 42--49.

\end{thebibliography}

\end{document}